\documentclass[lettersize,journal]{IEEEtran}
\usepackage{amsmath,amsfonts}
\usepackage{amsthm,amssymb}
\usepackage{diagbox}
\usepackage{xcolor, soul}
\usepackage{bbm}
\usepackage{algorithmic}
\usepackage{algorithm}
\usepackage{array}
\usepackage[caption=false,font=normalsize,labelfont=sf,textfont=sf]{subfig}
\usepackage{textcomp}
\usepackage{stfloats}
\usepackage{url}
\usepackage{verbatim}
\usepackage{graphicx}
\usepackage{cite}
\usepackage{hyperref}
\usepackage[normalem]{ulem}
\usepackage{cancel}

\newcommand{\mg}{\texttt{GMMF} }

\newtheorem{Def}{Definition}
\newtheorem{lem}{Lemma}
\newtheorem{rem}{Remark}
\newtheorem{thm}{Theorem}

\newtheorem{prop}{Proposition}
\newcommand{\E}{\mathbb{E}}
\newcommand{\Prb}{\mathbb{P}}

\newcommand{\cG}{\mathcal{G}}

\newcommand{\tA}{\widetilde{A}}
\newcommand{\tB}{\widetilde{B}}

\begin{document}
\title{Gotta match \textquotesingle em all: Solution diversification in graph matching matched filters}

\author{Zhirui Li, Ben Johnson, Daniel L. Sussman, Carey E. Priebe, Vince Lyzinski% <-

\markboth{Journal of \LaTeX\ Class Files,~Vol.~14, No.~8, August~2015}%
{Shell \MakeLowercase{\textit{et al.}}: Bare Demo of IEEEtran.cls for IEEE Journals}

\IEEEcompsocitemizethanks{\IEEEcompsocthanksitem Z.L. and V.L. are with the Dept. of Math., Univ. of Maryland, College Park, MD. E-mail:  \href{mailto:zli198@umd.edu}{zli198@umd.edu} (ZL); \href{mailto:vlyzinsk@umd.edu}{vlyzinsk@umd.edu} (VL).
\IEEEcompsocthanksitem B.J. is with Jataware Corp., Washington, DC. E-mail: \href{mailto:ben@jataware.com}{ben@jataware.com}.
\IEEEcompsocthanksitem D.L.S. is with the Dept. of Math. and Stat., Boston Univ., Boston, MA. E-mail: \href{mailto:sussman@bu.edu}{sussman@bu.edu}.
\IEEEcompsocthanksitem C.E.P. is with the Dept. of Applied Math. and Stat., Johns Hopkins Univ., Baltimore, MD. E-mail: \href{mailto:cep@jhu.edu}{cep@jhu.edu}.

\IEEEcompsocthanksitem This material is based on research sponsored by the Air Force Research
Laboratory (AFRL) and Defense Advanced Research Projects Agency (DARPA) under agreement number
FA8750-20-2-1001. 
The U.S. Government is authorized to reproduce and distribute reprints for Governmental purposes
notwithstanding any copyright notation thereon. The views and conclusions contained herein are
those of the authors and should not be interpreted as necessarily representing the official policies
or endorsements, either expressed or implied, of the AFRL and DARPA or
the U.S. Government.
The authors also gratefully acknowledge the support of the JHU HLTCOE.}}

\IEEEtitleabstractindextext{
\begin{abstract}
We present a novel approach for finding multiple noisily embedded template graphs in a very large background graph.
Our method builds upon the graph-matching-matched-filter technique proposed in Sussman et al. \cite{sussman2018matched}, with the discovery of multiple diverse matchings being achieved by iteratively penalizing a suitable node-pair similarity matrix in the matched filter algorithm. 
In addition, we propose algorithmic speed-ups that greatly enhance the scalability of our matched-filter approach. 
We present theoretical justification of our methodology in the setting of correlated Erd\H os-R\'enyi graphs, showing its ability to sequentially discover multiple templates under mild model conditions.  
We additionally demonstrate our method's utility via extensive experiments both using simulated models and real-world datasets, including human brain connectomes and a large transactional knowledge base.
\end{abstract}
\begin{IEEEkeywords}
statistical network analysis, graph matching, random graph models, subgraph detection
\end{IEEEkeywords}
}

\maketitle

 \IEEEdisplaynontitleabstractindextext

\section{Introduction}
Given two graphs $G_1=(V_1,E_1)\in \mathcal{G}_n$ and $G_2 = (V_2,E_2)\in \mathcal{G}_n$, where $\mathcal{G}_n$ represents the set of undirected, unweighted, and loop-free graphs with $n$ nodes, the Graph Matching Problem (GMP) aims to find the best possible alignments between nodes of $G_1$ and $G_2$ in order to minimize the edgewise structural differences. 
In its simplest form, the GMP seeks to minimize $\|A-PBP^T\|_F$ over $P\in\Pi_n$, where $A$ and $B$ are the adjacency matrices for $G_1$ and $G_2$ respectively, $\Pi_n$ represents the set of permutation matrices on $[n]:=\{1,2,\ldots,n\}$, and $\|\cdot\|_F$ denotes the matrix Frobenius Norm defined as $\|M\|_F = \sqrt{\sum_{j=1}^n\sum_{k=1}^m M_{jk}^2}$ for any $M\in \mathbb{R}^{n\times m}$. 
Note, we will often refer to a graph and its adjacency matrix interchangeably as in the case of real weighted edges they represent the same information.
The GMP is, in its most general form, equivalent to the NP-hard quadratic assignment problem. 
This computational complexity, combined with the problem's practical utility, has led to numerous approximation approaches to be posited in the literature; for detailed discussions on approximation algorithms, variants, and applications of the GMP, we refer the reader to the survey papers \cite{riesen2010exact}, \cite{yan2016short}, and \cite{conte2004thirty}.

An extension to the GMP is the Subgraph Detection Problem, also known as the Subgraph Matching Problem (SMP) or Subgraph Isomorphism Problem. This extension relaxes the assumption that both graphs have the same number of nodes. In essence, given $G_1\in \mathcal{G}_m$ and $G_2\in \mathcal{G}_n$ with $m<n$, the SMP aims to find a subgraph of $m$ nodes in $G_2$ that is structurally most similar to the template $G_1$. {Note that a general subgraph $G_s=(V_s, E_s)$ of a graph $G=(V,E)$ is defined such that $E_s\subseteq E$ and $V_s\subseteq V$. In contrast, the term induced subgraph imposes the additional condition that for any $i,j\in V_s$, $\{i,j\}\in E_s$ iff $\{i,j\}\in E$}.  This extension is significant in various applications. For instance, \cite{Zhang2017HiDDenHD} demonstrates the use of SMP on co-authorship networks to extract potential fake reviewers, while \cite{wu2009core} discusses SMP with known protein complexes and protein-protein interaction networks to identify new protein complexes. SMP is also valuable in analyzing brain neural networks, where it helps identify specific regions of interest across multiple networks for focused analysis, as shown in \cite{thies2004formal} and \cite{POWER2011665}. Additionally, the SMP has been used for activity template detection in large knowledge graphs \cite{moorman2018IEEE,moorman2021subgraph} among myriad other applications in machine learning, social network analysis, computer vision, and pattern recognition \cite{dulrahim1998graph, solnon2019experimental}. 

Numerous algorithms have been proposed to detect subgraphs from larger graphs that are isomorphic to the template (i.e., there exists {an induced} subgraph $G_2'$ and a permutation matrix $P\in\Pi_m$ such that $\|A-PB'P^T\|_F=0$, where $B'$ is the adjacency matrix of $G_2'$, and $A$ the adjacency matrix of $G_1$), with the first notable algorithm presented by \cite{ullmann1976algorithm}. 
Note that the (perhaps simpler) graph isomorphism problem also has a rich history in the literature, with recent results establishing at worst quasipolynomial complexity for the problem \cite{babai2015graph}.
Detailed explanations and comparisons between state-of-the-art algorithms can be found in survey papers such as \cite{sun2020memory,solnon2019experimental}. Recently, a series of papers \cite{moorman2018IEEE,moorman2021subgraph,yang2023structural} introduced an exhaustive (designed to find all subgraphs of $G_2$ isomorphic to $G_1$) tree-based/filtering method that reduces the time required for SMP by eliminating symmetries (referred to as ``structural equivalence" and ``candidate equivalence") within the graph.
The exhaustive nature of the tree-search/filtering based approaches is a key feature that will motivate our modification of the non-exhaustive algorithm of \cite{sussman2018matched} in the following section.
It should be noted that the aforementioned methods work well when an isomorphic copy of the template exists in the larger graph, but they often fail when such a promise is absent. 
In \cite{sussman2018matched,tu2020inexact}, the authors relax the isomorphism requirement and instead aim to find a subgraph that shares
the highest amount of structural (and feature-based in the case of \cite{tu2020inexact}) similarity with the template. 

Our focus in this paper will be the matched-filters-based approach (abbreviated \texttt{GMMF} for graph-matching matched filters) of \cite{sussman2018matched}, in which the authors adapt the Frank-Wolfe-based \cite{FW} \texttt{SGM} algorithm of \cite{ModFAQ} by proposing different padding techniques to ensure that the template has the same number of nodes as the larger graph. The validity of their proposed padding methods is supported by both extensive simulations and theoretical justification.
However, the \texttt{GMMF} algorithm and the adaptation in \cite{pantazis2022multiplex} (lifting the matched-filters approach to richly featured, multiplex networks) rely on efficiently solving iterative linear assignment problem {(LAP)} subroutines---via the Frank-Wolfe approach---which can be cumbersome in cases where the graphs are very large.
Moreover, these algorithms are not designed to exhaustively search the background graph for all close (but perhaps suboptimal) matches, aiming instead to find only the best fitting subgraph(s).  
In cases where more solution diversity is desired, this can limit the algorithm's applicability.
These two concerns motivate our extensions of the \texttt{GMMF} routine to allow for both more solution diversification and greatly enhanced scalability.
Note that code to implement this modified \texttt{GMMF} approach can be found at \url{github.com/jataware/mgmmf}.

We begin by introducing a random graph model in which we anchor our study, and provide an overview of the algorithm and modifications we will employ.

\subsection{Multiple Correlated Erd\H os-R\'enyi}
The Erdős-Rényi model \cite{ErdRen1963} is one of the most popular network models studied. 
While assuming all possible edges in the graph exist equally likely and independently, such a model still exhibits rich properties and provides fertile ground for studying graph matching problems. Discussions regarding thresholds 
of the graph properties for this model can be found in \cite{FK16,bollobas1998random} for the homogeneous Erd\H os-R\'enyi case and for the inhomogeneous case in \cite{bollobas07}.
Percolation theories on the Erdős-Rényi model have been proven in \cite{bhamidi2011first}.
Within the related correlated Erd\H os-R\'enyi model, sharp thresholds for graph de-anonymization are established in \cite{cullina2016improved,cullina2017exact,lyzinski2016information,wu2021settling}, and recent polynomial time algorithms for almost sure exact graph matching (i.e., recovering the optimal solution asymptotically almost surely) have been established in \cite{mao2023random} (with almost surely efficient seeded approaches proposed in \cite{mossel2020seeded}).
% The Erdős-Rényi model has found various applications in analyzing real-world data, as exemplified by \cite{gomez2006scale}.

Our present focus on recovering both optimal and near-optimal solutions in the \mg framework leads us to the following extension of the Erdős-Rényi model, dubbed the Multiple Correlated Erdős-Rényi model.
{This model is a natural extension of the classical correlated Erdős-Rényi model of \cite{Gross2013PGM,lyzinski2014seeded,wu2021settling,cullina2016improved,cullina2017exact} and the embedded template model of \cite{sussman2018matched}, where here we allow for multiple templates to be embedded in the background.
In practice (see Section \ref{sec:TKB}, it is often the case that there are multiple errorful copies of the template/motif in the background, and this model allows for such structure where the template matches have varying levels of noise.  }

\begin{Def} (Multiple Correlated Erdős-Rényi)
\label{def:mcer}
Let $m< n$ be nonnegative integers, with
$\Lambda^{(1)}\in [0,1]^{m\times m}$, 
$\Lambda^{(2)} \in [0,1]^{n\times n}$ probability matrices. 
Let $N$ be a nonnegative integer, and let $\mathcal{R}=(R_1,R_2,\cdots,R_N)$ be a sequence of symmetric matrices in $[0,1]^{m\times m}$. 
Two adjacency matrices, $A$ and $B$, follow the Multiple Correlated Erdős-Rényi Model with parameters $\Lambda_1$, $\Lambda_2$ and $\mathcal R$ if 
\begin{itemize}
    \item[i.] For all $u,v\in\{1,\cdots,m\}$, $A_{uv}\stackrel{ind.}{\sim}\text{Bernoulli}(\Lambda^{(1)}_{uv})$, and for all $u,v\in\{1,\cdots,n\}$, $B_{uv}\sim\text{Bernoulli}(\Lambda^{(2)}_{uv})$;
    \item[ii.] There exist induced subgraphs $(B^{(1)},\cdots,B^{(N)})$ of $B$, each with $m$ vertices {that are not necessarily disjoint among these subgraphs}, such that for $i=1,\cdots,N$, and $u,v\in\{1,\cdots,m\}$,
$A_{uv},B^{(i)}_{uv}\sim\text{Bernoulli}(\Lambda^{(1)}_{uv})$ and 
$$\text{correlation}(A_{uv},B^{(i)}_{uv})={R_{i_{uv}}};$$
\item[iii.] All edges in $B\setminus\{B^{(1)},\cdots,B^{(N)}\}$ are independent and are independent of all edges in $B^{(1)},\cdots,B^{(N)}$.
Furthermore, the collections $\{B^{(i)}_{uv}\}_{i=1}^N$ are independent of each other as $\{u,v\}\in\binom{\{m\}}{2}$ varies, {where $\binom{S}{j}$ denotes the collection of all possible subsets of order $j$ from the set $S$}.  Note that the edges within each collection $\{B^{(i)}_{uv}\}_{i=1}^N$ can have nontrivial dependence.
\end{itemize}
\end{Def}

\noindent In \cite{sussman2018matched}, the authors defined a similar Correlated Erdős-Rényi Model, which is a special case of our Multiple Correlated Erdős-Rényi model with the additional assumption that $N=1$.
Allowing $N$ to be greater than one allows us to embed multiple matches of the template $A$ into $B$, and vary the strengths of the matchings via $\mathcal{R}$.
Note that the structure of the multiple embeddings in Definition \ref{def:mcer} constrains $\Lambda_2$, as multiple copies of $\Lambda_1$ need to be embedded into $\Lambda_2$ as principle submatrices (up to reordering).

{As mentioned earlier, the proposed Multiple Correlated Erdős-Rényi Model extends the classic correlated Erdős-Rényi and the correlated Erdős-Rényi template models of \cite{lyzinski2014seeded,sussman2018matched}. 
This novel model explores the new, yet common phenomenon where multiple solutions (here, embedded templates) to the subgraph matching problem could potentially exist, and where certain nodes are more central to the template structure and should be preserved in multiple recovered templates. 
For example, consider a supply chain graph focusing on key suppliers. Major companies like Nvidia and Intel would be core nodes because their components are crucial across various products, while smaller or more specialized suppliers, whose products are limited to specific areas, would be less likely to exist in multiple templates. 
That said, Erdős-Rényi models are often not directly applicable for modeling real data, though correlated Erdős-Rényi models are a standard setting for graph matching theory in the literature.
Moreover, the proposed algorithms in this work perform well on real network data that is not Erdős-Rényi, see Sections \ref{sec:brain} and \ref{sec:TKB} for details.}

\subsection{Notations}
We will use the following asymptotic notations: for functions $f, g: \mathbb{Z}^+\rightarrow \mathbb{R}^+$
\begin{itemize}
    \item $f(n)=o(g(n))$, written as $f\ll g$,
    if $\lim_{n\rightarrow\infty} \frac{f(n)}{g(n)}=0$; 
    \item $f(n)=\omega(g(n))$, written as $f\gg g$,
    if $\lim_{n\rightarrow\infty} \frac{g(n)}{f(n)}=0$; 
    \item $f(n)=O(g(n))$, written as $f\lesssim g$,
    if $\exists C>0$ and $n_0$ such that $\forall n\geq n_0$, $f(n)\leq Cg(n)$;
    \item $f(n)=\Omega(g(n))$, written as $f\gtrsim g$,
    if $g(n) = O(f(n))$;
    \item $f(n)=\Theta(g(n))$ if $f(n)=O(g(n))$ and $g(n)=O(f(n))$.
\end{itemize}

{When the context is clear, for $j\in \mathbb{Z}^+$, we use $\binom{n}{j}$ to denote the binomial coefficient if $n$ is an integer, and $\binom{S}{j}$ to denote the collection of all subsets with $j$ elements if $S$ is a set.
For a given matrix $M$, the decomposition 
$$M=\bordermatrix{
    &j&k\cr 
    l& M^{11}&M^{12}\cr
    m&M^{21}& M^{22}}$$
    divides $M$ into four blocks where the numbers on the borders denote the corresponding dimensions. For example, $M^{11}$ is the upper-left $l\times j$ block of the matrix $M$. For a real number $\mathbf{r}$ and positive integers $m$ and $n$, $\mathbf{r}_m$ denotes the $m$-dimensional all $\mathbf{r}$ vector and $\mathbf{r}_{m\times n} $ denotes the $m\times n$ matrix with all entries equal to $\mathbf{r}$. For a square matrix $M$, $\operatorname{tr}(M)$ is the trace of $M$ defined as the sum of the diagonal entries of $M$. 
    }

\section{Solution diversification}
\label{sec:soldiverse}
In order to recover signals with suboptimal {$R_i$} structure (i.e., sufficiently entry-wise dominated by another {$R_j$}), our approach will make use of vertex-based graph features as done in \cite{pantazis2022multiplex}.  
These features will be represented in the form of a similarity matrix $S=S^{A,B}$ defined as follows.
\begin{Def}
\label{def:feat}
The similarity matrix between a pair of graphs $A\in \cG_m$ and $B\in \cG_{n}$ (where $m<n$) is a matrix $S=S^{A,B}\in [0,\infty)^{m\times n}$, where for $(i,j)\in\{1,\cdots,m\}\times \{1,\cdots,n\}$, we have $S^{A,B}_{ij}$ represents the similarity score between node $i\in V_A$ and $j\in V_B$. When the context is clear, we shall suppress the indices $A,B$ and simply write $S$.
\end{Def}

{In certain cases, when we have knowledge of the networks only limited to edge structures, there is not much we can do besides the standard matched filter approach. See Theorem \ref{thm:suss} below and \cite{sussman2018matched}. If we have labels or feature vectors for the nodes of the networks, we can in general try to find some proper distance measures to define our $S$ matrix. When no good distance measure can be defined, we can use the multiplex graph matching matched filters proposed in \cite{pantazis2022multiplex}.}

To incorporate the node similarities into our matching problem, we adapt the approach of \cite{sussman2018matched}. 
First, for integer $k$ we will let $J_k$ be the $k\times k$ hollow matrix with all off-diagonal entries equal to 1, and $\oplus$ {denote} the matrix direct {sum defined as $M_1\oplus M_2 = \begin{bmatrix}
    M_1 & \mathbf{0}\\
    \mathbf{0} & M_2
\end{bmatrix}$}.
Adopting an appropriate padding scheme:
\begin{itemize}
\item[i.] The \emph{centered padding} which matches $\tilde A=(2A-J_m)\oplus \mathbf{0}_{n-m,n-m}$ to $ \tilde B = 2B-J_n$; this seeks the best fitting induced subgraph of $B$ to match to $A$ according to the Frobenius norm GMP formulation. 
{As noted in \cite{pantazis2022multiplex}, this is equivalent to minimizing $\|A-PBP^T\|_F$ over $P\in\Pi_{m,n}$, where $\Pi_{m,n}=\{P\in\{0,1\}^{m\times n}$ s.t. $\mathbf{1}_m^TP\leq \mathbf{1}_n$, $P\mathbf{1}_n=\mathbf{1}_m\}$}.
\item[ii.] The \emph{naive padding} which matches $\hat A=A\oplus \mathbf{0}_{n-m,n-m}$ to $\hat B=B$; this seeks the best fitting subgraph of $B$ to match to $A$ where the objective to {minimize $\|\hat A-P\hat B P^T\|_F$ over $P\in\Pi_{n}$}.
\end{itemize}
Note that for any ${C,D}\in\mathcal{G}_n$
\begin{align*}
\text{argmin}_{P\in\Pi_n}\|C\!-\!PDP^T\|_F\!&=
\text{argmin}_{P\in\Pi_n} \|CP-PD\|_F\\
&=\text{argmax}_{P\in\Pi_n} \text{tr}(CPDP^T).
\end{align*}
{The above relation between the Frobenius form of the objective function and the trace form yields that the naive padding scheme is equivalent to maximizing $\text{tr}(APBP^T)$ over $P\in\Pi_{m,n}$.  
As in \cite{pantazis2022multiplex}, we see that the centered padding scheme is equivalent to minimizing $\|A-PBP^T\|_F$ over $P\in\Pi_{m,n}$, which in trace form is equivalent to maximizing $2\text{tr}(APBP^T)-\|PBP^T\|_F^2$ over $P\in\Pi_{m,n}$.
The extra $\|PBP^T\|_F^2$ term incorporates the penalty for edge/non-edge disagreements that distinguishes the centered from the naive padding.
We note here that while the above optimization could be cast as optimizing over $\Pi_{m,n}$, this makes the connection between the Frobenius norm and trace-form of the objective function a bit more nuanced across paddings (see also the discussion in \cite{pantazis2022multiplex}).  
We choose instead to present the optimization over the full permutation matrices $\Pi_n$, where naive padding is then finding $\text{argmax}_{P\in\Pi_n} \text{tr}(\hat AP\hat BP^T)$ and centered padding finding $\text{argmax}_{P\in\Pi_n} \text{tr}(\tilde AP\tilde BP^T)$ to ease exposition and highlight the connection between the two forms.}

{From the definition (see also the discussion in \cite{ModFAQ}), we see that the centered padding scheme penalizes like graph edit distance, with equal penalty for any extra or removed edges in the recovered templates. 
This is more useful if additional recovered structure is not desired in the recovered template.
The naive padding (which only rewards common edges, and does not penalize missing or extraneous edges) should be used if extra edges/activity in the background is unimportant, and the recovery of the template edges is the paramount task. 
As an example, consider matching the template $A\in\mathcal{G}_m$ that is an $m/2$ regular graph to the graph $B$, where $B$ is composed of two subgraphs connected by a single edge: one of which is a copy of $A$ with one missing edge (call this $B_A$), the other the complete graph on $m$ vertices (call this $B_K$).  Centered padding would match $A$ to $B_A$ and naive padding to $B_K$.
Note that if the graphs are weighted, the naive padding is more easily used, as the optimal centered padding scheme for weighted graphs is still an open research topic.}

Write $P=\begin{pmatrix}
P_{(1)}\\
P_{(2)}\end{pmatrix}$ where $P_{(1)}\in\mathbb{R}^{m\times n}$ and $P_{(2)}\in\mathbb{R}^{(n-m)\times n}$, we account for the similarity term by seeking the  solution to:
\begin{align*}
    &\underset{P \in \Pi_n}{\text{argmax}}& \underbrace{\operatorname{tr}\left(\tilde A P \tilde B P^{T}\right) + \lambda\operatorname{tr}\left(S{P_{(1)}^T}\right)}_{:=\tilde f(P,\lambda)}\,\,\,\begin{pmatrix}
\text{centered} \\
\text{padding}
\end{pmatrix}
\end{align*}
and
\begin{align*}
    &\underset{P \in \Pi_n}{\text{argmax}}& \underbrace{\operatorname{tr}\left(\hat A P \hat B P^{T}\right) + \lambda\operatorname{tr}\left(S{P_{(1)}^T}\right)}_{:=\hat f(P,\lambda)}\,\,\,\begin{pmatrix}
\text{naive} \\
\text{padding}
\end{pmatrix}
\end{align*}
where $\lambda$ is a hyperparameter chosen/tuned by the user{, $S$ is the similarity matrix as defined in Definition \ref{def:feat} and $P_{(1)}$ is the matrix consisting only the top $m$ rows of the matrix $P$}.
The \mg approach then uses multiple random restarts of the following procedure to search $B$ for the best fitting subgraphs to $A$.  We will present the algorithm{, adapted from \cite{ModFAQ,FAQ}, incorporating the gradient of the feature term while ignoring the seeded portion},  in the centered padding case, the naive padding setting following mutatis mutandis.
\begin{itemize}
\label{alg:MultipleGMMF}
\item[1.] Apply centered padding to $A$ and $B$ yielding $\tilde A$ and $\tilde B$;
\item[2.] Considering $N_{mc}$ random restarts,\\
for $k=1,2,\cdots,N_{mc}$, do the following
\begin{itemize}
\item[i.] Set initialization $P^{(0)}=\gamma\mathbf{1}_k\mathbf{1}_k^T/n+(1-\gamma) P$ where $P\sim\text{Unif}(\Pi_n)$ and $\gamma\sim$Unif[0,1];
\item[ii.] While $\|P^{(t)}-P^{(t-1)}\|_F>\eta$ for a specified tolerance $\eta>0$, do the following
\begin{itemize}
\item[a.] Compute the gradient 
$$\nabla_P\tilde f(P^{(t)},\lambda)=\tilde A^T P^{(t)}\tilde B+\tilde A P^{(t)}\tilde B^T+\lambda{ I_{n\times m}}S$$
{where $I_{n\times m}$ is the matrix consisting with the first $m$ columns of an $n\times n$ identity matrix.}
\item[b.] Compute search direction 
$$Q^{(t)}=\text{argmax}_{Q\in \mathcal{D}_n}\text{tr}\left[\nabla_P\tilde f(P^{(t)},\lambda)^\top Q\right];$$
{via the Hungarian Algorithm proposed in \cite{hungarian}} 
where $\mathcal{D}_n$ is the set of $n\times n$ doubly stochastic matrices;
\item[c.] Perform line search in the direction of $Q^{(t)}$ by solving
$$\gamma^*=\text{argmax}_{\gamma\in[0,1]} \tilde f(\gamma P^{(t)}+(1-\gamma)Q^{(t)},\lambda)$$
{This step involves optimizing a quadratic function of $\gamma$, and an analytical solution is obtained by taking the derivative to find the critical point, followed by comparing the function values at the two boundary points and the critical point.}
\item[d.] Set $P^{(t+1)}=\gamma^* P^{(t)}+(1-\gamma^*)Q^{(t)}$
\end{itemize}
\item[iii.] Set
$P^{(*,k)}=\max_{P\in \Pi_n}\text{tr}\left(P^\top P^{(\text{final})}\right);$
\end{itemize}
\item[3.] Rank the recovered matchings $\{P^{(*,k)}\}_{k=1}^{N_{mc}}$ by largest to smallest value of the objective function $\tilde f(P,\lambda)$; output the ranked list of matches
\end{itemize}

In the above algorithm, we can steer the algorithm away from previously recovered solutions (in the random restarts) by biasing the objective function away from these already recovered solutions. 
Suppose that the $k$-th random restart returns the solution $P^{(*,k)}$ (with corresponding permutation $\sigma^{(*,k)}$).
To accomplish this, we define the mask
$M^{k,\epsilon}\in\mathbb{R}^{m,n}$ via
$$
M^{k,\epsilon}_{ij}=
\begin{cases}
    (1-\epsilon)&\text{ if }j=\sigma^{(*,k)}(i);\\
    1&\text{ else}.
\end{cases}
$$
As an example, consider $[n] = \{1,2,\ldots,n\}$; $\sigma^{(\ast,1)}$ maps $[3]\mapsto[7]$ by fixing $[3]$ identically; and $\sigma^{(\ast,2)}$ maps $[3]\mapsto[7]$ by $\sigma^{(\ast,2)}(1)=1$, $\sigma^{(\ast,2)}(j)=j+3$ for $j=2,3$. Then 
\begin{align*}
    &M^{1,\epsilon} = \left[\begin{matrix}1-\epsilon&1&1&{\mathbf{1}_{1\times 4}}\\1&1-\epsilon&1&{\mathbf{1}_{1\times 4}}\\1&1&1-\epsilon&{\mathbf{1}_{1\times 4}}\end{matrix}\right]\\
    &M^{2,\epsilon} = \left[\begin{matrix}1-\epsilon&{\mathbf{1}_{1\times 3}}&1&1&1\\1&{\mathbf{1}_{1\times 3}}&1-\epsilon&1&1\\1&{\mathbf{1}_{1\times 3}}&1&1-\epsilon&1\end{matrix}\right]
\end{align*}

In the next random restart, we apply the mask to the current similarity matrix $S^{(k,\epsilon)}$ via
$S^{(k+1,\epsilon)}=M^{k,\epsilon}\circ S^{(k,\epsilon)}$ (note: $S^{(1,\epsilon)}=S$) where ``$\circ$'' represents the matrix Hadamard product.
Considering the previous example, let $S$ be a $3\times 7$ matrix, then
\begin{align*}
    S^{(3,\epsilon)}& = M^{1,\epsilon}\circ M^{2,\epsilon}\circ S\\
    &= \left[\begin{smallmatrix}
    (1-\epsilon)^2S_{11} &S_{12} &S_{13} &S_{14}&S_{15}&S_{16}&S_{17}\\
    S_{21} &(1-\epsilon)S_{22}&S_{23} &S_{24}&(1-\epsilon) S_{25}&S_{26}&S_{27}\\
    S_{31} &S_{32}&(1-\epsilon)S_{33} &S_{34}&S_{36}&(1-\epsilon) S_{36}&S_{37}\end{smallmatrix}\right]
\end{align*}

After penalization, we then seek to solve 
\begin{align}
    &\underset{P \in \Pi_n}{\text{argmax}}& \!\!\!\!\underbrace{\operatorname{tr}\!\left(\tilde A P \tilde B P^{T}\right) \!+\! \lambda\operatorname{tr}\!\left(S^{(k+1,\epsilon)}{P_{(1)}^T}\right)}_{:=\tilde f_{\epsilon,k+1}(P,\lambda)}\begin{pmatrix}
\text{centered} \\
\text{padding}
\end{pmatrix}\label{eq:cpgmp}\\
    &\underset{P \in \Pi_n}{\text{argmax}}& \!\!\!\!\underbrace{\operatorname{tr}\!\left(\hat A P \hat B P^{T}\right) \!+\! \lambda\operatorname{tr}\!\left(S^{(k+1,\epsilon)}{P_{(1)}^T}\right)}_{:=\hat f_{\epsilon,k+1}(P,\lambda)}\begin{pmatrix}
\text{naive} \\
\text{padding}
\end{pmatrix}\label{npgmp}
\end{align}

{The masks} effectively slightly down-weight the similarity scores for the previously recovered matrices.  Note that an overly draconian choice of $\epsilon\approx 1$ may have the effect of steering the algorithm away entirely from recovered solutions, and might not allow for overlapping solutions to be returned.  This is suboptimal in the case where a few key edges/vertices are expected to appear in many recovered templates.

{Our algorithm uses a simple gradient descent-based optimization, which is computationally fast when compared to more complex approximations to NP-hard GMP solution. 
To avoid the pitfalls of first-order methods (e.g., local maxima) that could lead to sub-optimal solutions, we, in practice, run the algorithm multiple times (easily parallelized) with random starting points as our Monte Carlo simulations.
This allows us to better explore the objective function and leverages the speed of each Frank-Wolfe iterate.}
\begin{rem}
    We can also apply the mask $M^{k,\epsilon}$ directly to the gradient or to the initialization in the \mg algorithm outlined above.
    In the gradient penalizing case,  step (a.) for the $(k+1)$-st random restart becomes
    \begin{align*}
        &\nabla_P^{(k+1),\epsilon}\tilde f(P^{(t)},\lambda)\\
        &=M^{1,\epsilon}\circ 
        % \textcolor{red}{\cancel{M^{2,\epsilon}\circ}}
        \cdots\circ M^{k,\epsilon}\circ\left(\tilde A^T P^{(t)}\tilde B+\tilde A P^{(t)}\tilde B^T+\lambda {I_{n\times m}}S\right)
    \end{align*}
    In the initialization penalizing, the mask is directly applied to $P^{(0)}$ followed by rescaling to ensure double stochasticity. {Similar ideas of penalizing the gradient to diversify solutions when solving optimization problems, particularly to find weaker or flatter optimizers, exist in the literature (e.g., \cite{barrett2021implicit,zhao2022penalizing}).}
    We will consider only the similarity penalization in the theory below {because this approach can be easily incorporated into the analysis of the graph matching objective function without need for delving into the optimization steps.}
    We also note that the penalty constant $\epsilon$ {could} be replaced by a sequence of penalties $\{\epsilon_k\}$ if we are expecting high level of template overlap in the embedded templates.
    {Yet another approach would be to penalize the edge weights of the recovered templates in the large network. Assuming a high correlation between edge structures and node similarities, the result should be similar to our approach. However, the efficiency of penalizing edges could be suboptimal due to the significantly higher number of edges compared to nodes in the denser regime.}
\end{rem}

\subsection{Theoretical benefits of down-weighting}
\label{theory}
We next provide theoretical justification for the down-weight masking in the setting where there are two overlapping embedded templates in the background (i.e., where $N=2$ in Definition \ref{def:mcer}).
The case where $N>2$ follows from repeated applications of the case where $N=2$, as does the case of no template overlap.
In the case where we expect to find only one recovered template, \cite{sussman2018matched} provides a detailed characterization of methods and conditions for detecting such a recovered template with high probability (sans similarity $S$). {Throughout the rest of the manuscript, all graphs and parameters should be indexed by the number of nodes in the larger graph, $n$. To improve readability, we will suppress the subscript $n$ whenever the context is clear. }

We consider {M}ultiple {C}orrelated Erd\H os-R\'enyi graphs with the following structure.
We will consider $\Lambda_1=pJ_m$ and $\Lambda_2=pJ_n$. {Note that our theories can be easily extended to inhomogeneous Erd\H os-R\'enyi graphs by using 0 and 0.25 as the lower and upper bounds of the variance of any Bernoulli random variables. 
While in practice, almost no network is purely Erd\H os-R\'enyi, such models are particularly useful for theoretically studying matchability phase transitions, and are a standard setting for deriving graph matching theoretical results \cite{lyzinski2014seeded,cullina2016improved,cullina2017exact,wu2021settling}.
Indeed, homogeneous Erd\H os-R\'enyi provides a difficult theoretical setting as there is no heterogeneity correlation across graphs 
\cite{fishkind2019alignment}, and the matching signal is entirely contained in the edge-correlation $R$.
If there was signal in both the edge structure $\Lambda$ and in $R$, as is common in most real-world networks, the practical difficulty would likely be reduced significantly as additional statistics (e.g., degree sequence, graph bottlenecks, centrality, etc.) could be more easily leveraged to match the template. Here, we consider}
\begin{align*}
    &A=\bordermatrix{
    &m-k&k\cr 
    m-k& A^{11}&A^{12}\cr
    k&(A^{12})^T& A^{22}};\\
 \end{align*}   
 \vspace{-0.5in}
\begin{align*}
B = \bordermatrix{
    & 2m-k &n-2m+k\cr
    2m-k & C & D\cr
    n-2m+k &D^T& E}
\end{align*}
where 
$$
C=\bordermatrix{
    &m-k&k&m-k\cr 
    m-k& C^{11}&C^{12}&C^{13}\cr
    k& (C^{12})^T&C^{22}&C^{23}\cr
    m-k& (C^{13})^T&(C^{23})^T&C^{33}\cr}
$$
{Further, for real numbers $0<r_3<r_2<r_1<1$,}
% {\sout{(where for a square matrix $M$, $\triangle(M)$ refers to the upper triangular portion of $M$)}}
\begin{align*}&\text{corr}(A_{ij},B_{h\ell})=\\
&{\begin{cases}
    r_1 &\text{ if }i=h,\,j=\ell,\,i\leq m,\,j\leq m-k\\
    r_1 &\text{ if }i=h,\,j=\ell,\,i\leq m-k,\,j\leq m\\
    r_2 &\text{ if }i=h,\,j=\ell,\,m-k<i\leq m,\,m-k<j\leq m\\
    r_3 &\text{ if }i=h-m,\,j=\ell-m,\,i\leq m-k,\,j\leq m-k\\
    r_3 &\text{ if }i=h-k,\,j=\ell-m+k,\,m-k<i\leq m,\\
    &\hspace{3mm}\,j\leq m-k\\
    r_3 &\text{ if }i=h-m+k,\,j=\ell-k,\,i\leq m-k,\\
    &\hspace{3mm}\,m-k<j\leq m\\
    0&\text{ else }
\end{cases}}
\end{align*}
This structure ensures that the two embedded copies of the template (each of size $m$) have a non-trivial overlap of size $k>0$.
The overlap again is designed to model the case where there are key vertices in the background that appear in multiple template embeddings.
Moreover, the case where $k=0$ is conceptually and theoretically simpler than the $k>0$ case and follows as a corollary to the theory below.

{Observe that with shuffling channels $Q^{(A)}\!\!\in\!\!\Pi_m$ and $Q^{(B)}\!\!\in\!\Pi_n$ applied to matrices $A$ and $B$, our model can account for Multiple Correlated Erd\H os-R\'enyi graphs such that two embeded templates exist with a shared sub-region of all kinds, aligning with the assumptions that more than one solutions exist and certain nodes must be preserved in all solutions.}

We will consider the centered padding scheme, where we match $\widetilde A=(2A-J_{m})\oplus \mathbf{0}_{n-m,n-m}$ to $\widetilde B=(2B-J_{n})$.
Analogous results can be derived for the naive scheme, which we leave to the reader.
In this setting, we will consider $S$ of the form
{$$S=\bordermatrix{
    &m-k&k&m-k&n-2m+k\cr 
    m-k& S^{11}&S^{12}&S^{13}&S^{14}\cr
    k&S^{21}& S^{22}&S^{23}&S^{24}}$$}
    where all entries of $S$ are independent and bounded, without loss of generality bounded in $[0,1]$ (e.g., Beta distributed), random variables, and where 
    \begin{align*}
    &\text{ the diagonal elements of }S^{11}\text{ have mean }\mu_1\\
        &\text{ the diagonal elements of }S^{22}\text{ have mean }\mu_2\\
        &\text{ the diagonal elements of }S^{13}\text{ have mean }\mu_3
    \end{align*}
    and all other entries have mean $\mu_4$. 
Here we will assume that 
$\mu_1>\mu_2>\mu_3>\mu_4$.
Let $P^*$ map $A$ to 
$$\bordermatrix{
    &m-k&k\cr 
    m-k& C^{11}&C^{12}\cr
    k& (C^{12})^T&C^{22}}
$$
and $\tilde P$ maps $A$ to 
$$
\bordermatrix{
    &m-k&k\cr 
    m-k& C^{33}&(C^{23})^T\cr
    k& C^{23}&C^{22}\cr},
$$
so that of the two embedded templates, 
the embedding via $P^*$ is stronger (in that the correlation is higher entry-wise as are the similarity scores on average) than that provided by $\tilde P$.

The goal is to down-weight/penalize the strongly embedded template so that the optimal solution to Eq. \ref{eq:cpgmp} is $\tilde P$ as opposed to $P^*$.
The features here are the key, as without this ability to down-weight the features (or the gradient) we do not expect to find $\tilde P$ by solving Eq. \ref{eq:cpgmp}.
Indeed, if we can only observe the edges of $A$ and $B$, without any further information provided by $S$ the results of \cite{sussman2018matched} provides the following theorem:
\begin{thm}
\label{thm:suss}
Let $A$ and $B$ be graphs as described above. Assuming we can only observe the edges of $A$ and $B$ { 
 but have no additional knowledge about the vertex-based graph features}, then with probability at least $1-n^{-2}$, we have that $\underset{P \in {\Pi_n}}{\operatorname{argmax}} \operatorname{tr}\left(\tilde A P \tilde B P^T\right)=\{P^*\}$.
\end{thm}
If the strongly embedded template is penalized, then $S$ is weighted as follows.  
Here, $S^{(2,\epsilon)}$ is set to be (where ``$\circ$'' is the matrix Hadamard product)
\begin{align*}
{\tiny 
{\bordermatrix{
    &m\!\!-\!\!k&k&m\!-\!k&n\!-\!2m+k\cr 
    m-k\!\!\!\!& (\mathbf{1}_{m-k,m-k}\!-\!\epsilon I_{m-k})\circ S^{11}&S^{12}&S^{13}&S^{14}\cr
    k\!\!\!\!&S^{21}& (\mathbf{1}_{k,k}\!-\!\epsilon I_{k})\circ S^{22}&S^{23}&S^{24}}}}
    \end{align*}
    We next state our main result, which is for the Multiple Correlated Erd\H os-R\'enyi model where the entries of $S$ are independent, and bounded (in $[0,1]$).  Note the proof can be found in Appendix \ref{app:pf}.
\begin{thm}
\label{thm:main}
Let $A$ and $B$ be two graphs constructed as above. If there is a constant $\alpha\in [1/2,1)$ such that
\begin{itemize}
    \item[i.] $m-k=\Theta(m)$;
     $m^{1-\alpha}=\omega(\log^4 n)$
    \item[ii.] $\lambda=m^{\alpha}$;
    \item[iii.] $(r_1-r_3)\ll m^{\alpha-1}$; $r_3,r_2,r_1$ are bounded away from $0$ and $1$;
    \item[iv.] $\mu_3>(1-\varepsilon)\mu_1$ and $(1-\varepsilon)\mu_2>\mu_4$; the differences $\mu_3-(1-\varepsilon)\mu_1$, $(1-\varepsilon)\mu_2-\mu_4$, and $\mu_3-\mu_4$ are bounded away from 0;
    \item[v.] $p$ is bounded away from $0$ and $1$;
\end{itemize}
then if $\widetilde \Pi$ is the set of permutations perfectly aligning the weakly embedded template (i.e., of the form $\tilde P\oplus Q$), we have
$$\mathbb{P}(\text{argmax}_{P\in\Pi}{\tilde f}_{\varepsilon,1}(P,\lambda)\subset \widetilde \Pi)\geq 1-e^{-\omega(\log n)}.$$
\end{thm}

Note that when $\lambda=0$, the objective function becomes the standard one considered in \cite{sussman2018matched}, and Theorem \ref{thm:suss} applies. However, when $\lambda>0$ is sufficiently large, the feature similarity becomes crucial in the objective function. By increasing $\varepsilon$, we force the global optimizer to move away from the first recovered template, thus expecting to recover a different in-sample subgraph.
When $\lambda$ is too large, the noise in the $S^{(2,\epsilon)}$ matrix (provided by the entries with mean $\mu_4$) can swamp the signal and suboptimal recovery is possible.
Note also that there are myriad combinations of parameter growth conditions under which Theorem 1 (or analogues) will hold, and our theorem is not claiming full generality.  We will not mine these conditions further herein. {We point out here that $p$ as a function of $n$ being bounded away from $0$ implies that the graph we consider here is dense, as in much of the other graph matching literature. Additionally, we could encode the diameter of the graphs into our objective function to eliminate disconnected matches, if any occur.}

{By properly realizing the asymptotic notation of $\omega$ as a fixed function and thus expressing $m$ in terms of $\log n$, we can easily translate our results into a finite sampling result. Note that most papers discussing graph matchabilities in the literature prove their results in an asymptotic context, meaning that the asymptotic behavior tells us what to expect for large enough networks, as seen in real data.}

\begin{rem}
Beyond independent entries for $S$, we could consider $S$ as a similarity between vertex features.
One possible approach to define such a similarity matrix $S$ is via well-constructed distance functions. 
In these cases, a function of two variables $s(a,b)$ (resp., $d(a,b)$ for a distance/dissimilarity function, where, for example, we could then define $s(a,b)=1/d(a,b)$) is defined such that $s$ increases (resp., $d$ decreases) as $|a-b|$ becomes smaller, e.g., $s(a,b) = e^{-|a-b|}$. For each node $i\in V_1=\{1,2,\cdots,m\}$ and $j\in V_2=\{1,2,\cdots,n\}$, if we model random vertex features via $X_i,Y_j\sim F$, we can then define  $S_{ij} = s(X_i,Y_j)$ for all $i,j$.
Note that Theorem \ref{thm:main} {relies only on the fact that the means of the Beta random variables are well-separated, so it} can be easily adapted to account for these vertex-dependent similarities, as long as the {scores satisfy certain tail decay conditions and the} expectation of similarity scores for node pairs under different recovery schemes are bounded away from each other. 
\end{rem}

We have developed an approach for recovering weaker signal versions of the embedded templates with high probability, under mild conditions. 
Our approach builds upon the method introduced in \cite{sussman2018matched}, where we initially recover the strongest in-sample template. Subsequently, we select a sufficiently large value for $\lambda$ and gradually increase $\varepsilon$ until we identify an additional in-sample template.
An important aspect is that once a second recovered template is found, we can iterate the algorithm by penalizing the similarity for both recovered templates. This allows us to discover more templates until we have exhausted all possibilities or until the penalty coefficient term for overlapping signal in unrecovered templates becomes too small compared to the noise in the non-signal nodes.

We next proceed to further demonstrate the validity of our proposed method via experiments in the Multiple Correlated Erdős-Rényi model and via two real data experiments.

\section{Experimental results}
\label{sec: experiment}

Our proposed modification of the \hyperref[alg:MultipleGMMF]{\mg algorithm} offers a simple and efficient solution for solution diversification. 
Moreover, incorporating a multiple of the  $\operatorname{tr}\left(S{P_{(1)}^T}\right)$ term into the existing Frank-Wolfe iterations allows us to efficiently incorporate {the} feature information.
Next, we will address scalability issues arising from the line search in steps (ii.a) and (ii.b) of the \hyperref[alg:MultipleGMMF]{\mg algorithm}.
To tackle this, note that the gradient $\nabla_P\tilde f(P^{(t)},\lambda)$ forms an $m\times n$ matrix. 
In practice, we often have $m\ll n$ (i.e., $m$ is significantly smaller than $n$). 
In the linear assignment search step (ii.b), each vertex in $A$ can only be assigned to one of the vertices in $B$ with the top $m$ values in the gradient. 
Therefore, for the allowable matchings for vertex $i$ in $A$, we compute the $m$ largest entries of row $i$ of $\nabla_P\tilde f(P^{(t)},\lambda)$.
All other row $i$ entries can be discarded. After performing this partial sorting operation for each of the $m$ rows (each row costing $O(n)$ time using the \texttt{Introselect} algorithm \cite{musser1997introspective})
the resulting matrix of allowed assignments is at most $m\times m^2$ and the Hungarian algorithm \cite{hungarian} applied to the Linear Assignment Problem (LAP) on this rectangular matrix is of complexity $O(m^4)$ \cite{tarjan}. 
Moreover, solving the LAP on this reduced matrix ensures the same result as solving it on the full $\nabla_P\tilde f(P^{(t)},\lambda)$. This approach effectively reduces the complexity of the LAP solver subroutine from $O(n^3)$ to $O(mn+m^4)$, and we observe substantial speedups in practice. 
We note that this complexity reduction technique was also discussed in \cite{bijsterbosch2010solving}. {The discussion above focuses only on the LAP subroutine of the \hyperref[alg:MultipleGMMF]{\mg algorithm} since it is often a computational bottleneck in the procedure. }

Due to the intractability of computing the exact graph matching solution in all but small cases,
in a few of the experiments below we make use of seeded vertices in the graph matching subroutine.
Seeded vertices are those whose correct alignment is known a priori.
In this case (assuming for the moment that the seeding maps vertices $\{1,2,\ldots,s\}$ in $A$ to $\{1,2,\ldots,s\}$ in $B$), the graph matching seeks to optimize $\tilde f(P,\lambda)$ over $P$ of the form $P=I_s\oplus Q$ (or $\hat f(P,\lambda)$ in the naive padding case) using the \texttt{SGM} algorithm proposed in \cite{ModFAQ}. 
In practice, seeds are often expensive to compute---indeed in the real data experiment on template detection in knowledge graphs below, we have no seeded vertices---though a few seeds can often lead to dramatically increased performance.
In the simulations below, the seeded approach helps overcome the computational intractability of the matching subroutine and are quite useful for demonstrating the utility of our solution diversification step.

\begin{figure}[t!]
    \centering
    \includegraphics[width = 0.47\textwidth]{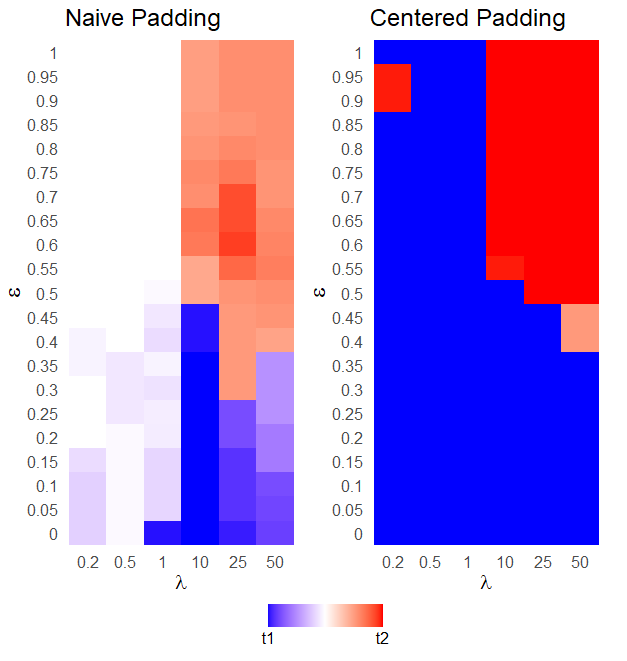}
    \caption{We fix $k=10$ and use the seeded \mg  algorithm to match $A$ with $B$ using 5 seeds randomly selected from the overlapping nodes of $B^{(1)}$ and $B^{(2)}$ as described in Section \ref{sec:N2}.
    We plot the recovering results over $\varepsilon$ (here $\varepsilon$ is used to penalize the stronger of the two embedded templates) and $\lambda$, averaged by 20 Monte-Carlo simulations. In the figures, stronger colors represent better recovery of the embedded templates, and t1 (blue) stands for template 1, t2 (red) stands for template 2, with white squares corresponding to the case when none of the two templates was recovered {or equal amounts of each template were recovered among the 20 simulations}.}
    \label{fig:N2Both}
\end{figure}

\subsection{Two overlapping templates}
\label{sec:N2}

Our first experiment verifies the proposed algorithm in our multiple correlated Erd\H os-R\'enyi model with $N=2$. 
In our setup, we take $\Lambda^{(1)} = \mathbf{0.8}_{50\times 50}, \Lambda^{(2)} = \mathbf{0.8}_{500\times 500}$.
{
To generate graphs with tractable edge correlations, we create an Erdős-Rényi background graph with edge probability $u$. 
We then sample induced graphs from the background graph by flipping the existence of each edge independently with probability $v$. Appendix 7.2 of \cite{li2023clustered} provides a formula for the correlation value as a function of the edge existence probability of the background graph, $u$, 
and the edge flipping probability, $v$. {Choosing $v$ to be $0.0074, 0.0168, 0.0326$ and $0.0495$ yields the correlation values used in our follow-up experiments as approximately $0.954, 0.897, 0.803$ and $0.706$, respectively (due to computational precision, these are estimates of $0.95, 0.9, 0.8, 0.7$ resp.).}
Now let} 
\begin{align*}
    R_1=\bordermatrix{
    &m-k&k\cr 
    m-k& R_1^{11}&R_1^{12}\cr
    k&(R_1^{12})^T& R_1^{22}};
\end{align*}
\begin{align*}
    R_2 = \bordermatrix{
    & k &m-k\cr
    k & R_1^{22} & R_2^{12}\cr
    m-k &(R_2^{12})^T& R_2^{22}}
\end{align*}
    where $R_1^{11}$ and $R_1^{12}$ are matrices with all entries set to 0.954; $R_2^{12}$ and $R_2^{22}$ are matrices with all entries set to $0.803$; and $R^{22}_1$ is a matrix with all entries set to 0.897. 
    Note the overlap in $R_1$ and $R_2$ is there to make sure the induced subgraphs $B^{(1)}$ and $B^{(2)}$ have $k$ overlapping nodes. 
    For the similarity matrix $S$, we set $$S=\bordermatrix{
    &m-k&k&m-k&n-2m+k\cr 
    m-k& S^{11}&S^{12}&S^{13}&S^{14}\cr
    k&S^{21}& S^{22}&S^{23}&S^{24}}$$
    where all entries of $S$ are independent Beta random variables, such that 
    \begin{align*}&\text{ the diagonal elements of }S^{11}\sim \operatorname{Beta}(\text{{$\mu_1=0.6$}});\\
    &\text{ the diagonal elements of }S^{22}\sim \operatorname{Beta}(\text{{$\mu_2=0.55$}});\\
    &\text{ the diagonal elements of }S^{13}\sim \operatorname{Beta}(\text{{$\mu_3=0.5$}})\end{align*}
    and all other entries are sampled from $\operatorname{Beta}(\text{{$\mu_4=0.1$}})$. Note that for Beta distribution {with parameters $\alpha, \beta$}, {we have }$\mu = \frac{\alpha}{\alpha+\beta}$. We randomly sample $\alpha\sim U(0,1)$ and use the specified $\mu$ to calculate the corresponding $\beta$. Other combinations of parameters are also explored and plotted, see Appendix \ref{app:N2}.

   We fix $k=10$ (see Appendix \ref{app:N2}  for the case of $k=15,40$), and use the seeded \mg algorithm with 5 seeds randomly selected from the overlapping nodes of $B^{(1)}$ and $B^{(2)}$. 
    In Figure \ref{fig:N2Both} we use the naive padding (left) and the centered padding (right) and plot results over numerous choices of $\varepsilon$ (here $\varepsilon$ is used to penalize the stronger of the two embedded templates) and $\lambda$, averaged over 20 Monte-Carlo simulations.
    In the figures, stronger colors represent better recovery of the embedded templates, and t1 (blue) stands for template 1, t2 (red) stands for template 2, with white squares corresponding to the case when none of the two templates was recovered {or equal amounts of each template were recovered among the 20 simulations}. 
    From the figures, we see that when $\varepsilon$ is small we recover the stronger embedded template $(B_1)$, and as $\varepsilon$ increases we move away from the stronger embedded template, and---provided a suitable value of $\lambda$---we successfully recover the weaker embedded template as desired.
    From the plots, we can see the centered padding outperforms the naive padding for recovering the second template in this multiple correlated Erd\H os-R\'enyi model, which aligns with the results proven in \cite{sussman2018matched}.
    The phenomenon is clearer for larger $k$ (see, for example, Figure \ref{fig:bigK} in the Appendix), where the naive padding detects either only template 1 or nothing (denoted by the white in the plot) for whatever $\lambda$ we choose.

\begin{figure}[t!]
    \centering
    \includegraphics[width = 0.4\textwidth]{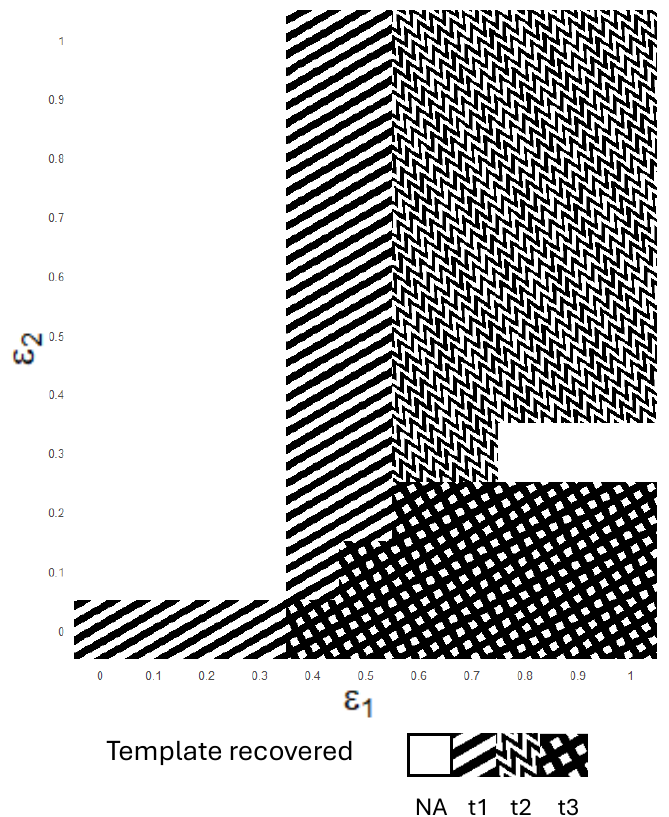}
    \caption{We fix $k=10,\lambda=25$ and use the seeded \mg algorithm with the centered padding to match $A$ with $B$ using 5 seeds randomly selected from the overlapping nodes of $B^{(1)}, B^{(2)}$ and $B^{(3)}$, where $B^{(1)}, B^{(2)}$ and $B^{(3)}$ are induced subgraph of $B$ such that graphs $A$ and $B$ follows multiple correlated ER model as described in Section \ref{sec:N3}. We plot the recovering results over $\varepsilon_1$ (penalty applied to the diagonal elements of $S^{(11)}, S^{(22)}$) and $\varepsilon_2$ (penalty applied to the diagonal elements of $S^{(13)}, S^{(22)}$), averaged by 20 Monte-Carlo simulations. 
In the figure, the different patterns represent which template was recovered (in majority): t1 for template 1, t2 for template 2, and t3 for template 3, with white squares corresponding to the case when none of the three templates was recovered.}
    \label{fig:N3}
\end{figure}
\subsection{Three overlapping templates}
\label{sec:N3}
To better illustrate the iterative feature of the proposed algorithm, we construct a multiple correlated Erd\H os-R\'enyi model with $N=3$ and apply our algorithm to attempt to recover all of the three embedded templates. 
In this experiment, we take $\Lambda^{(1)} = \mathbf{0.8}_{50\times 50}, \Lambda^{(2)} = \mathbf{0.8}_{500\times 500}$,  
\begin{align*}
    R_1=\bordermatrix{
    &m-k&k\cr 
    m-k& R_1^{11}&R_1^{12}\cr
    k&(R_1^{12})^T& R_1^{22}};
\end{align*}
\begin{align*}
R_2 = \bordermatrix{
    & k &m-k\cr
    k & R_1^{22} & R_2^{12}\cr
    m-k &(R_2^{12})^T& R_2^{22}}
\end{align*}
    and $$R_3 = \bordermatrix{
    & k &m-k\cr
    k & R_1^{22} & R_3^{12}\cr
    m-k &(R_3^{12})^T& R_3^{22}},$$ 
    where $R_1^{11}$ and $R_1^{12}$ are matrices with all entries set to 0.954; $R_2^{12}$ and $R_2^{22}$ are matrices with all entries set to $0.803$; 
    $R_3^{12}$ and $R_3^{22}$ are matrices with all entries set to $0.706$; and $R^{22}_1$ is a matrix with all entries set to 0.897. 
    Note again this structure ensures the induced subgraphs $B^{(1)}, B^{(2)}$ and $B^{(3)}$ have exactly $k$ pairwise overlapping nodes, and to make the $B^{(1)}$ a better probabilistic match (i.e., a stronger embedding) than $B^{(2)}$ which is a better probabilistic match than $B^{(3)}$.
    
   For the similarity matrix $S$, we set $S$ to be
   \begin{align*}
    \bordermatrix{
    &m-k&k&m-k&m-k&n-3m+2k\cr 
    m-k& S^{11}&S^{12}&S^{13}&S^{14}&S^{15}\cr
    k&S^{21}& S^{22}&S^{23}&S^{24}&S^{25}}
    \end{align*}
    where all entries of $S$ are independent Beta random variables, such that 
    \begin{align*}
    &\text{ the diagonal elements of }S^{11}\sim \operatorname{Beta}(\mu = 0.7)\\
        &\text{ the diagonal elements of }S^{22}\sim \operatorname{Beta}(\mu=0.6)\\
        &\text{ the diagonal elements of }S^{13}\sim\operatorname{Beta}(\mu=0.55)\\
        &\text{ the diagonal elements of }S^{14}\sim\operatorname{Beta}(\mu=0.5)
    \end{align*}
    and all other entries are sampled from $\operatorname{Beta}(\mu=0.1)$. Again for all the Beta distributions, we randomly sample $\alpha\sim U(0,1)$ and use the specified $\mu$ to calculate the corresponding $\beta$.
    
In Figure \ref{fig:N3}, we fix $k=10, \lambda=25$ (see Appendix \ref{app:N3} for other combinations of $(k,\lambda)$), and use the seeded \mg algorithm with the centered padding (the naive padding behaved sub-optimally for recovering template 3, see Appendix \ref{app:N3} for details) and 5 seeds randomly selected from the overlapping nodes of $B^{(1)}, B^{(2)}$ and $B^{(3)}$. We plot the recovering results over $\varepsilon_1$ (penalty applied to the diagonal elements of $S^{11}, S^{22}$) and $\varepsilon_2$ (penalty applied to the diagonal elements of $S^{13}, S^{22}$), averaged by 20 Monte-Carlo simulations. 
In the figure, the different patterns represent which template was recovered (in majority): t1 for template 1, t2 for template 2, and t3 for template 3, with white squares corresponding to the case when none of the three templates was recovered. {Note that when $\varepsilon_1\in [0,0.3]$, we have not recovered template 2, thus it would be impossible to have a second penalty term associated with $\varepsilon_2$, the corresponding parts are left white and should not be interpreted. }
As we can see, when both $\varepsilon_1, \varepsilon_2$ are small, we recover the strongest embedded template; as $\varepsilon_1$ increases, we move away from the strongest embedded template to the second strongest embedded template; finally when both $\varepsilon_1, \varepsilon_2$ get large enough, we recovered the third embedded template as desired.

% %
\begin{figure}[t!]
    \centering
    \includegraphics[width=0.4\textwidth]{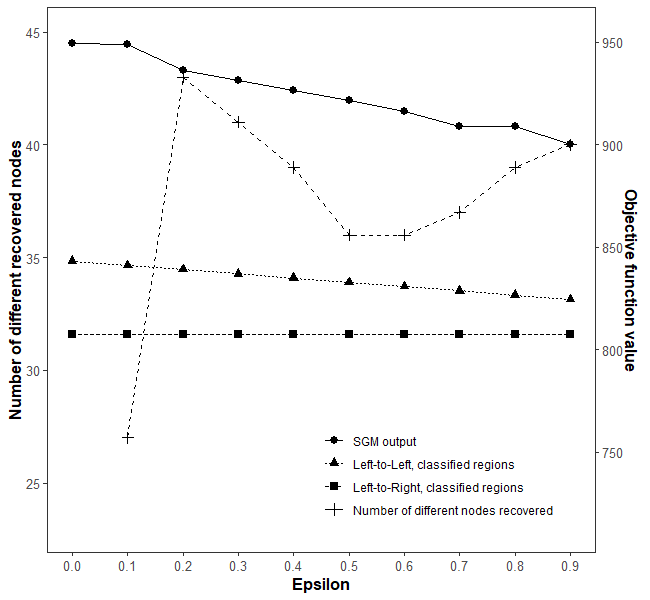}
    \caption{We run our proposed algorithm using the seeded \mg algorithm with 500 restarts and 5 seeds selected from the node pairs $\{(s_j,w_j)\}_{j=1}^6$ as described in \ref{sec:brain}, taking the result with highest objective function value (Eq. \ref{npgmp}) as the output. For each $\varepsilon$, we compute the GM objective function value (left axis) of the resulting matrix with the template; we also computed the objective function value with respect to the alignment given by the template to the same classified brain region in the left hemisphere in $B$ (Left--to--Left in the plot), as well as the objective function value given by the template to the symmetric region from the right hemisphere in $B$ (Left--to--right in the plot). Also for $\varepsilon>0$, we calculated the number of novel nodes recovered in each matching compared to the subgraph detected with $\varepsilon=0$ (right axis).}
    \label{fig:MRI}
\end{figure}
\subsection{MRI Brain data}
\label{sec:brain}
We now apply the proposed algorithm to a real data set of human connectomes from \cite{zuo2014open}, where we consider the BNU1 test-retest connectomes processed via the pipeline at \cite{kiar2018high}
(see \url{http://fcon_1000.projects.nitrc.org/indi/CoRR/html/bnu_1.html} and \url{https://neurodata.io/mri/} for more detail).
The dataset consist of test-retest DTI data for each patient processed into connectome graphs.
Moreover, the brain graphs here are segmented into regions of interest and contain $(x,y,z)$ DTI coordinates. 
We chose patient ``subj1" for our present experiment. In 
the first connectome, which contains 1128 nodes, we select a region (region 30) with 46 nodes from the left hemisphere to act as our template. 
We then extracted the subgraph induced by the selected region and considered it as our graph $A$. The entire graph of the other brain scan from the same patient, which contains 1129 nodes, was designated as our graph $B.$
Our goal is to recover both the same region of interest in the left hemisphere (we consider this the ``strong'' embedded template) and the corresponding region (region 65) in the right hemisphere (the ``weak'' embedded template).

To construct the similarity matrix, we consider the $(x,y,z)$ coordinates of each node in the processed MRI scan from graphs $A$ and $B$. 
Subsequently, we randomly selected 12 nodes, denoted as $s = \{s_1, \ldots, s_{12}\}$, from graph $A$. For nodes ${s_1, \ldots, s_6}$, we identified the corresponding nodes ${w_1, \ldots, w_6}$ from the same region, same hemisphere, in graph $B$. 
Whereas for nodes ${s_7, \ldots, s_{12}}$, we identified the corresponding nodes ${w_7, \ldots, w_{12}}$ from the corresponding region in the other hemisphere relative in graph $B$. 
Note, these are not seeded vertices, but are simply used to construct the similarity matrix $S$.
Informally, we consider each pair of $(s_i, w_i)$ as a ``bridge'' which has distance 0; allowing us to define a suitable distance across hemispheres for any nodes $u\in A$ and $v\in B$ as follows.  
Setting the distance between corresponding seeded nodes across hemispheres to be $0$, we define the distance via
$$d(u,v)=\min_{j\in [12]}\{\|u-s_i\|_2+\|v-w_i\|_2\}$$
Here, $\|\cdot\|_2$ represents the standard Euclidean distance. We defined our similarity matrix such that $S_{ij} \propto -d(u_i, v_j)$.

We executed our proposed algorithm using the seeded \mg algorithm with 500 restarts, with 5 seeds selected out of the node pairs $\{(s_j,w_j)\}_{j=1}^6$, so the seeds exist in the left-hemisphere only; and we selected the result with the highest objective function value (in Eq. \ref{npgmp} using naive padding with $\lambda = 0.1$ as the final match; see appendix \ref{app:brain} for the case of $\lambda=1$). 
Naive padding worked well with the irregular structure of the brain networks here, and we are actively researching whether centered or naive padding is more appropriate in non edge-independent models.
For each $\varepsilon$, (plotted in Figure \ref{fig:MRI}) we compute the GM objective function value (right axis) of the resulting matrix with the template; we also computed the objective function value with respect to the alignment given by the template to the same classified brain region in the left hemisphere in $B$ (Left--to--Left in the plot), as well as the objective function value given by the template to the symmetric region from the right hemisphere in $B$ (Left--to--Right in the plot). Also for $\varepsilon>0$, we calculated the number of novel nodes recovered in each matching compared to the subgraph detected with $\varepsilon=0$ (left axis, ``Number of different nodes recovered'' in the plot).  
As expected, the objective function value obtained from the output of the seeded \mg algorithm is better than the ground truth alignment.
Furthermore, by increasing $\varepsilon$ beyond 0.1, we observed a deviation from the original recovered template, leading to the discovery of a new subgraph matching the template close to optimally. 
We comment that the decrease in the objective function value based on the alignment provided by the classified brain regions across the scans is a result of the seeds in the \texttt{SGM} algorithm, where the similarity scores between these 5 seeds pairs decrease as $\varepsilon$ increases.

{We close this example mentioning that by judiciously encoding neuronal information via the feature similarity matrix, the performance of the template recovery increases dramatically.
The power of the similarity formalism is that it enables incorporation of any feature information for which similarities can be computed.
For example, in the knowledge graph example of Section \ref{sec:TKB} the similarity encodes both numeric/quantitative and semantic/qualitative (i.e., ontological) features together.
We note however that an adversarial $S$ could potentially break our approach, as it could violate our working assumption of positively correlated edge structures and node similarities. }

\subsection{Template discovery in TKBs}
\label{sec:TKB}
For our second real data example, we consider the transactional knowledge base (TKB) of \cite{purohit2021transactional}.  
The graph is constructed from a variety of information sources including news articles, Reddit, Venmo, and bibliographic data.
Moreover, nodes and edges are richly attributed. 
Node attributes include a unique node ID, node type (according to a custom ontology), free text value, entry ID (used to identify the node in the Wikidata Knowledge Base), date and latitude/longitude.
Edge attributes include a unique edge ID, edge type (according to a custom ontology), and edge argument (providing additional edge information).
See \cite{purohit2021transactional} for more information on the construction of this network and details on the custom ontological structure.

Along with the large background graph, \cite{purohit2021transactional} describes the creation of multiple signal templates (with varying levels of noise) to search for in the background. In addition to perfectly aligned templates (i.e., background subgraphs isomorphic to the template), templates are embedded with different and varying noise levels, necessitating noisy template recovery. 

The full graph has 14,220,800 nodes and 157,823,262 edges.  For each template, we do some simple preprocessing of the graphs that reduces their size (for instance, removing node types that do not appear in the template, removing dangling edges, etc).  This preprocessing yields the pruned graphs that are fed into our matching algorithm, which have approximately $13 \times 10^6$ nodes and $32 \times 10 ^ 6$ edges.
As in \cite{pantazis2022multiplex}, we create a multiplex network from this TKB by dividing edge types (from the different sources) and different ontological edge types into multiple weighted graph layers (weighted based on a measure of ontological similarity), and using node features to define a node--to--node similarity matrix. 
Note that we use naive padding here, as the edge structure is naturally weighted and optimal centering in the weighted case is nuanced and the subject of present study.
The multiplex adaptation of the \mg procedure can be found in \cite{pantazis2022multiplex}, and amounts to adapting the Frank-Wolfe approach to the objective functions
\begin{align*}
\underset{P \in \Pi_n}{\text{argmax}} {\sum_i}\operatorname{tr}\Big(\Big[A^{(i)}\oplus &0_{n-m,n-m}\Big] P B^{(i)} P^{T}\Big) \\
+&\lambda\operatorname{tr}\left(S{P_{(1)}^T}\right)
\end{align*}
where $A^{(i)}$ (resp., $B^{(i)}$) represent the template structure (resp., background structure) in layer $i$ of the multiplex graph and where nodes with a common label across layers are assumed aligned.

We show the effect of the solution diversification via the following experiment.  
To measure the fidelity of recovered templates when no ground truth is present, we use the graph edit distance (GED) metric outlined in \cite{ebsch2020using}.
We run 32 random restarts of the \mg algorithm for each template recovery (we plot results for template 1A, 1B, 1C, 1D here; results for templates 2 and 3 can be found in Appendix \ref{app:tkb}), plotting the empirical CDF of the GED of the recovered templates; results are plotted in Figure \ref{fig:tkb1}.
Different penalization values are represented with different colors in the plot.
\begin{figure*}[t!]
    \centering
    \includegraphics[width=0.8\textwidth]{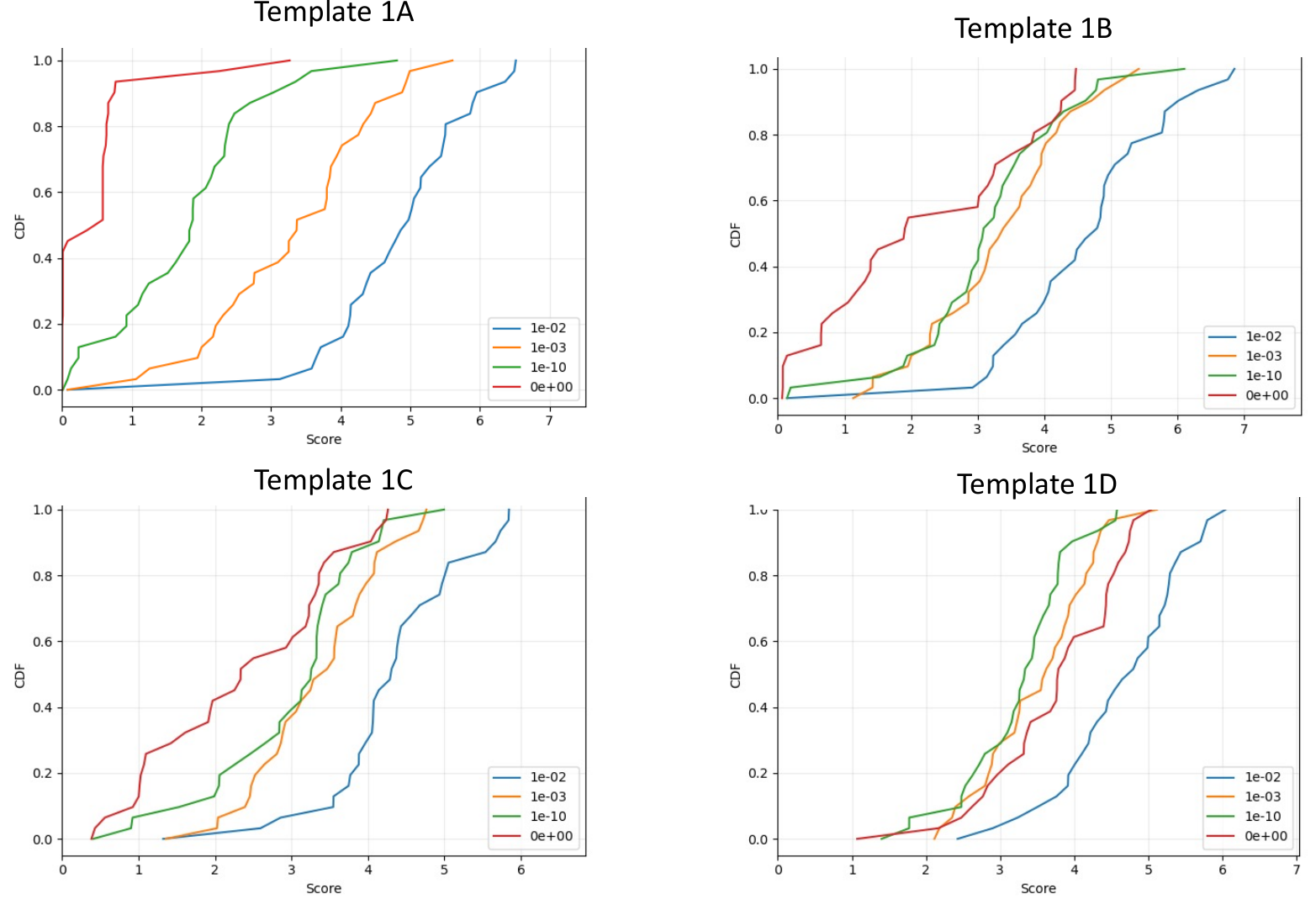}
    \caption{We run 32 random restarts of the \mg algorithm for each template recovery, plotting the empirical CDF of the GED of the recovered templates.
Different penalization values are represented with different colors in the plot.}
    \label{fig:tkb1}
\end{figure*}
Note that in each template, version A has an isomorphic copy in the background while this is not guaranteed for the other templates (as they have noise introduced in the embeddings).
From Figure \ref{fig:tkb1}, we see that the solution diversification is successful at yielding recovered templates including our optimal fits (the best recovered templates in the $\epsilon=0$ case) and templates that are close to optimal by GED.  These templates further recover significant signal not recovered in the $\epsilon=0$ case; see Table \ref{tab:temp1}.
In the table, for each vertex in the template we show how many vertices in the background have similarity greater than 0 (i.e., are potential matches)---this is shown in the \#$>0$ row.
We see that the solution diversification is successful at recovering additional possible matches in the suboptimal recovered templates.
We also see that too severe of a penalty (here the $\epsilon=0.01$ case) yields fewer unique nodes and worse GED fits. {In summary, when there is no penalty, we consistently detect only a limited number of certain templates. When a penalty term is applied, we successfully recover some new templates. However, as the penalty becomes too large after a certain number of iterations, the noise in the non-signal nodes starts to dominate the signal in the uncovered templates, resulting in some sub-optimal recoveries.} While choosing an optimal $\epsilon$ is of paramount importance, we do not have a fully principled recommendation for a best choice.
We do recommend smaller penalty combined with more random restarts which achieved our best results. {We close this example by noting that, because very few nodes in the background have a similarity greater than 0 to nodes in our templates, it is possible to filter the background by removing other nodes. However, the combinatorics after filtering remain complex; see \cite{moorman2021subgraph} for more detailed discussions. 
That said, this filtering could be used to ``soft seed" our graph matching (see \cite{fang2018tractable}). 
We would prefer the soft seeding (where seeds are used to initialize the matching but not fixed throughout), as it is possible that for some nodes in the background, the similarity is incorrectly calculated as 0 but the matching should still be made, making a hard seed filtering step perhaps less desirable.}

\begin{table*}[b!]
\centering
\begin{tabular}{c|ccccccccccccccccccccccccccccccccccccccccccccccccccccc}
node&1&2&3&4&5&6&7&8&9&10&11&12&13&14&15&16&17&18&19&20&21&22&23\\\hline
 $\epsilon=0$& 1 & 8&  1&  3 &20&  1&  1&  1 & 3&  1& 14&  2&  1&  1&  5&1&  1&  1&  1 &2&  3& 1&  3 \\ \hline
 $\epsilon=0.01$& 1 &23&  2& 15 &32&  1&  1&  1 & 8&  1& 24&  5&  4&  1& 24& 1&  1&  1&  1 &3& 16& 1& 14 \\ \hline
 $\epsilon=10^{-3}$& 4 &21& 20& 29 &32&  5&  5&  4 &27&  1& 26&  5& 27&  3& 31&1&  4&  4&  3 &3& 27& 5& 26 \\ \hline
 $\epsilon=10^{-10}$& 4 &32& 30& 32 &31&  5&  5&  4 &32&  5& 30&  5& 32&  4& 32&4& 20&  4&  3 &3& 32& 5& 32 \\ \hline\hline
  \#$>0$& 4 &32& 32& 32 &32&  5&  5&  4 &32&  5& 32&  5& 32&  4& 32&4& 32&  4&  3 &3& 32& 5& 32 \\ \hline\hline\hline
node& 24&25&26&27&28&29&30&31&32&33&34&35&36&37 \\\hline
 $\epsilon=0$&19&  1&  1&  1&  3&  4&  1&3&  4&  1&  2&  4&  1&  1  \\ \hline
 $\epsilon=0.01$&32&  1&  2&  2&  5& 18&  1& 3& 25&  1&  2&  5&  1&  1 \\ \hline
 $\epsilon=10^{-3}$&31& 12&  5&  4&  5& 27&  4&13& 31&  4& 15&  5&  5&  6 \\ \hline
 $\epsilon=10^{-10}$&32& 29&  5&  4&  5& 32&  4&28& 32&  4& 31&  5&  5& 21 \\ \hline\hline
 \#$>0$&32& 32&  5&  4&  5& 32&  4&32& 32&  4& 32&  5&  5& 32 
 \end{tabular}
 \caption{In Template 1A, we list for each node how many vertices in the background have similarity greater than 0 (i.e., are potential matches)---this is shown in the \#$>0$ row).  We then show for the different penalization levels, how many of these possible matches were recovered across the 32 random restarts.}
 \label{tab:temp1}
 \end{table*}

\section{Conclusion and discussion}
In this paper, we have introduced a workflow for iteratively identifying multiple instances of noisy embedded templates within a large graph. Our approach extends the matched-filters-based method for noisy subgraph detection by considering both the edgewise structure and node feature similarities. By incorporating these factors, we have achieved a more diversified and scalable approach to effectively uncover embeddings of noisy copies of graph templates.
The theoretical analysis of our algorithm demonstrates that, under the assumption of a strong correlation between the edgewise structural similarities and node-wise feature similarities, our approach can successfully identify multiple embedded templates within a large network. To validate the effectiveness of our proposed workflow, we conducted experiments using simulations based on {the} Multiple {Correlated} Erd\H os R\'enyi models, as well as real-world data sets such as human brain connectomes and the TKB dataset.

Furthermore, we present several intriguing questions that merit further exploration. In particular, the manuscript assumes an agreement between edge-structural similarity and node feature similarity. It would be valuable to investigate scenarios where such an agreement is absent, specifically identifying sharp parameter thresholds that lead to 
edge-structure dominated recovery, node-feature dominated recovery, and mixed-effect recovery. In all three cases, it is crucial to establish robust measures for evaluating the correctness of edge-wise matching between the template and the recovered template.
Additionally, we highlight the issue of overlapping nodes between two embedded templates. The reliability of our algorithm relies on the ratio of overlapping parts between the templates being moderate. However, if the ratio is excessively high, penalizing already recovered templates may lead to sub-optimal results. To address this concern, it would be beneficial to develop methods that specifically target penalization on the non-overlapping regions while preserving the signal of the overlapping region, thereby enhancing the algorithm's performance.

\vspace{2mm}
\bibliographystyle{plain}
\bibliography{biblio.bib}

\vspace{-8mm}
\begin{IEEEbiography}[{\includegraphics[width=1in,keepaspectratio]{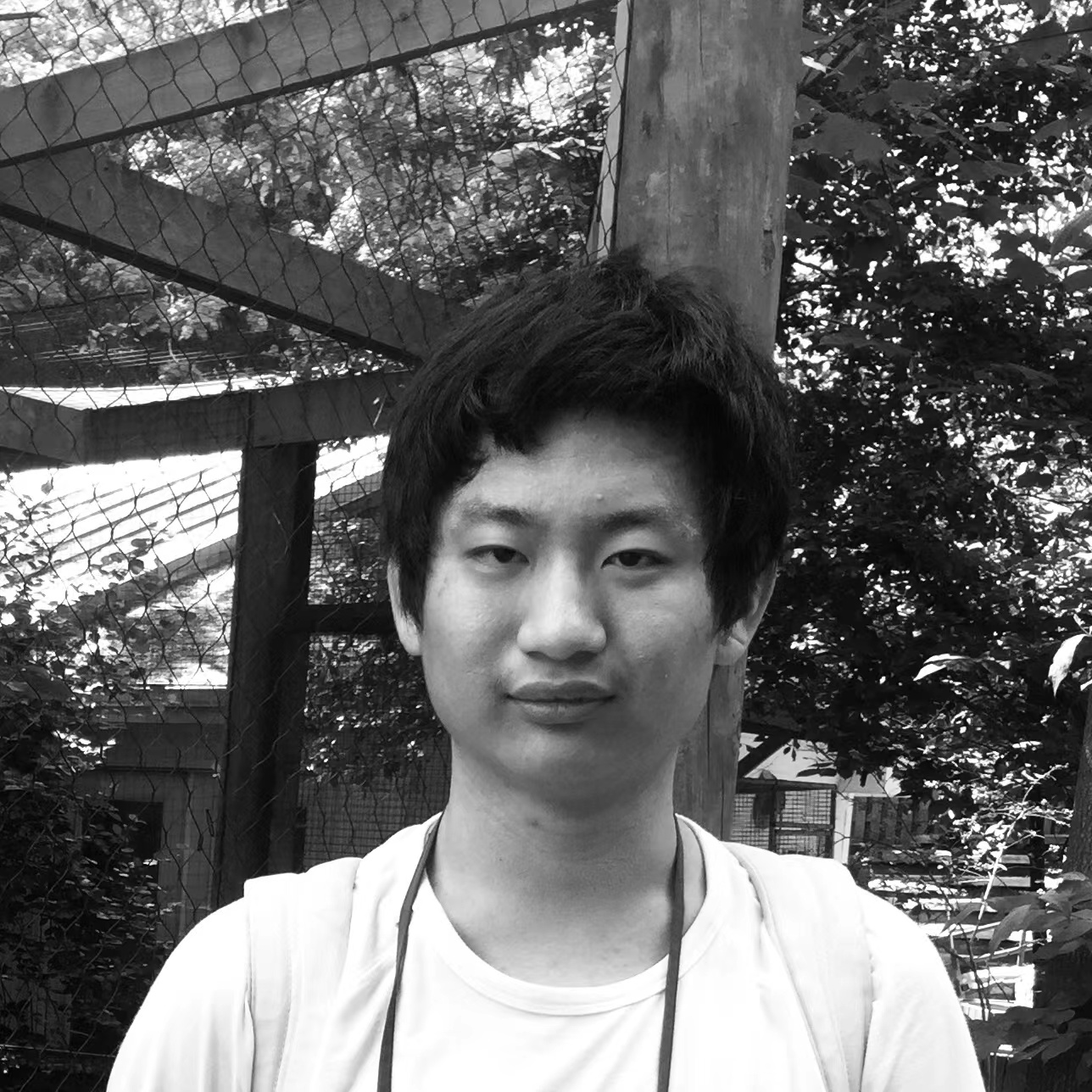}}]{Zhirui Li}
 received his BS degrees in Mathematics and Statistics from the University of Iowa in
2020. He is currently a doctoral student in the AMSC program at the University of Maryland, College Park, working under the supervision of Dr. Vince Lyzinski. His research interests include graph matching, stochastic processes, numerical optimization and statistical machine learning.
\end{IEEEbiography}

\vspace{-8mm}
\begin{IEEEbiography}[{\includegraphics[width=1in,keepaspectratio]{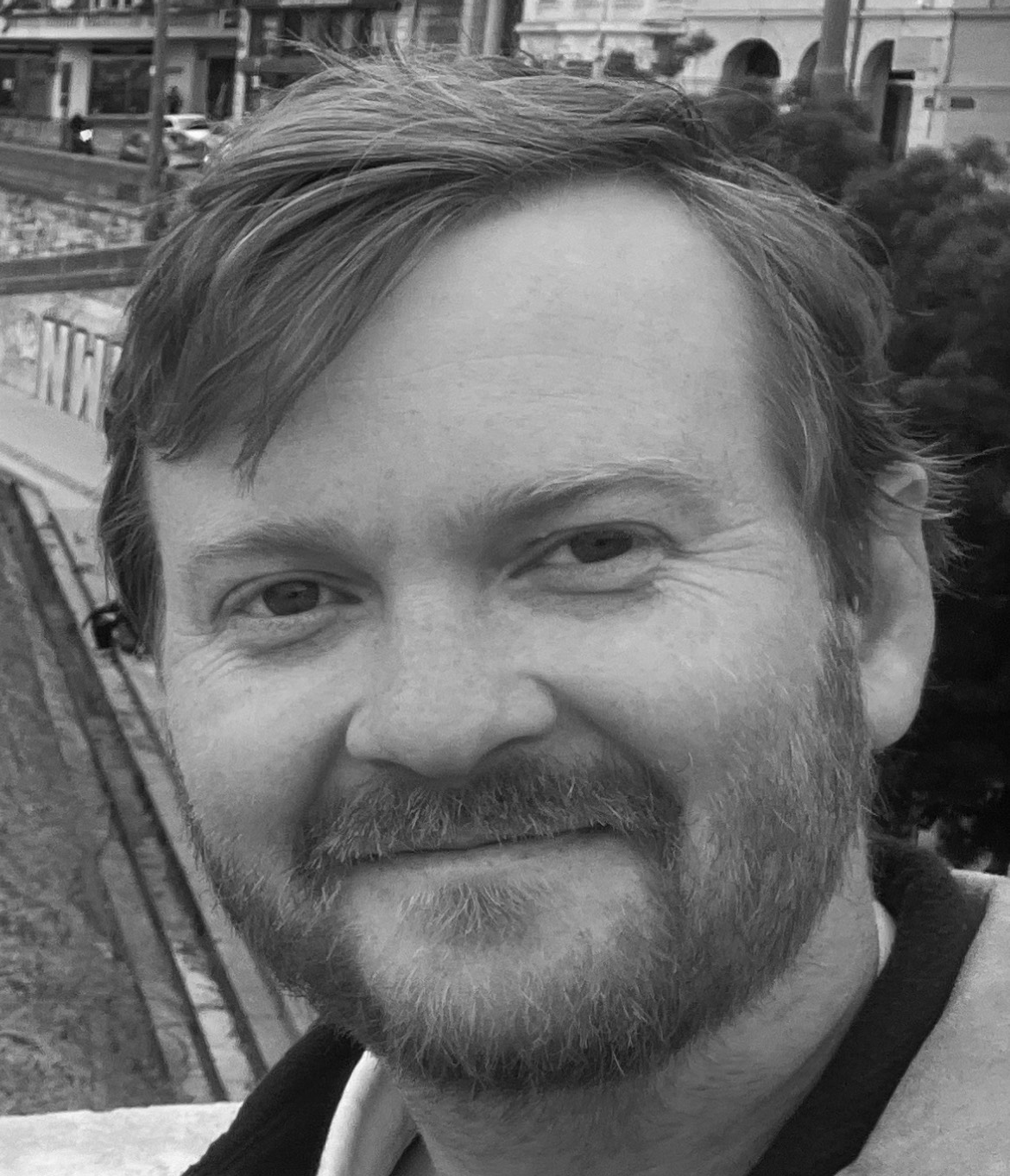}}]{Ben K Johnson} received a AB degree in mathematics from Bowdoin College in 2012 and a MA in statistics from Yale University in 2013.  He is Chief Scientist of Jataware Corp where he focuses on both fundamental and applied research in AI/ML.
\end{IEEEbiography}

\vspace{-8mm}
\begin{IEEEbiography}[{\includegraphics[width=1in,keepaspectratio]{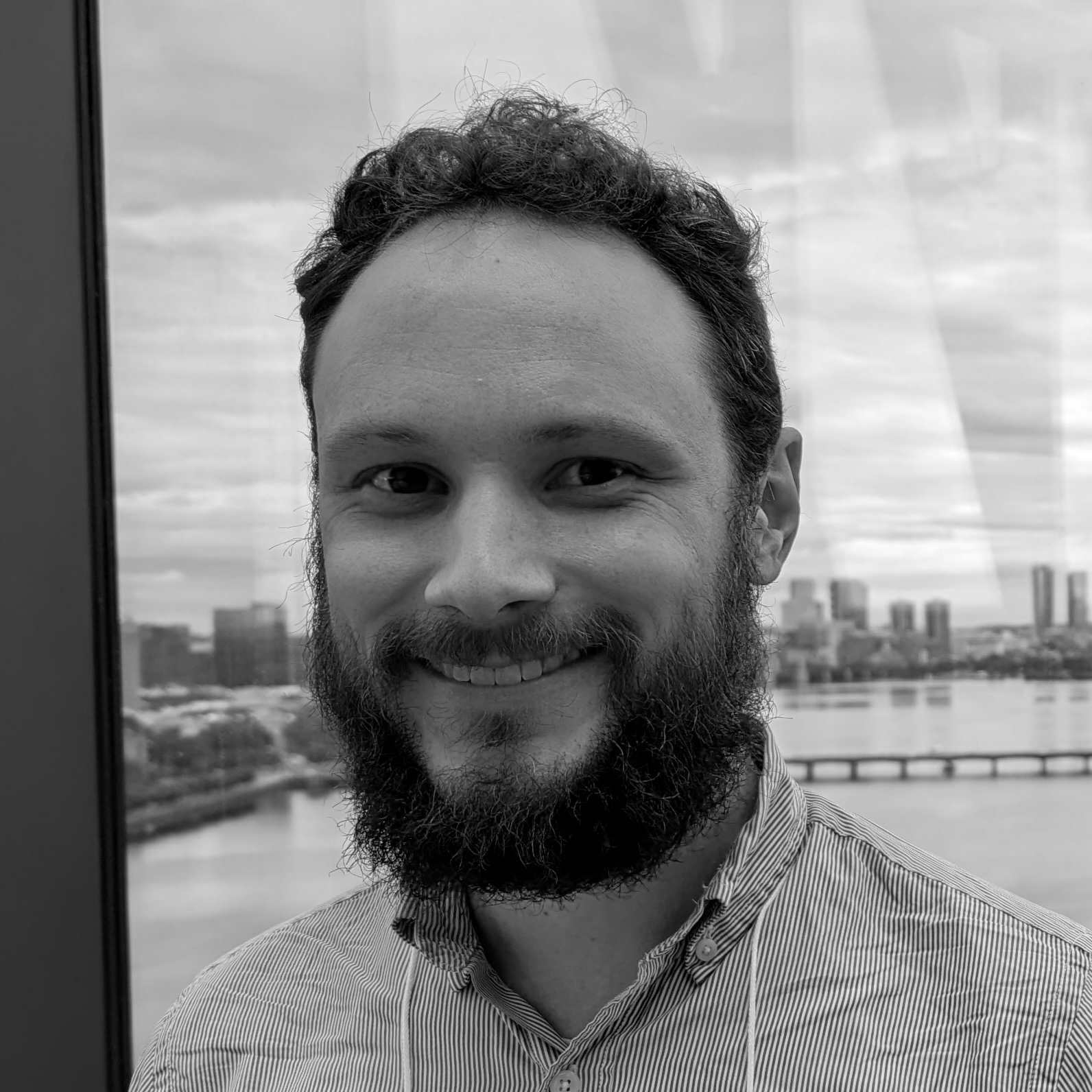}}]{Daniel L. Sussman} received the B.A. degree in mathematics from Cornell University, Ithaca, NY, USA, in 2008, and the Ph.D. degree in applied mathematics and statistics from Johns Hopkins University, Baltimore, MD, USA, in 2014. From 2014 to 2016, he was a Postdoctoral Fellow with Harvard University Statistics Department, Cambridge, MA, USA. He is currently an Assistant Professor with the Mathematics and Statistics Department, Boston University, Boston, MA, USA. His research interests include spectral methods for graph inference, multiple graph inference, graph matching, connectomics, and causal inference under interference.
\end{IEEEbiography}

\vspace{-8mm}
\begin{IEEEbiography}[{\includegraphics[width=1in,keepaspectratio]{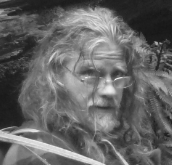}}]{Carey E. Priebe} received the BS degree in mathematics from Purdue University in 1984, the MS degree in computer science from San Diego State University in 1988, and the PhD degree in information technology (computational statistics) from George Mason University in 1993. From 1985 to 1994 he worked as a mathematician and scientist in the US Navy research and development laboratory system. Since 1994 he has been a professor in the Department of Applied Mathematics and Statistics at Johns Hopkins University, where he is director of the Mathematical Institute for Data Science (MINDS). He is a Senior Member of the IEEE, an Elected Member of the International Statistical Institute, a Fellow of the Institute of Mathematical Statistics, and a Fellow of the American Statistical Association.
\end{IEEEbiography}

\vspace{-8mm}
\begin{IEEEbiography}[{\includegraphics[width=1in,height=1.25in,clip,keepaspectratio]{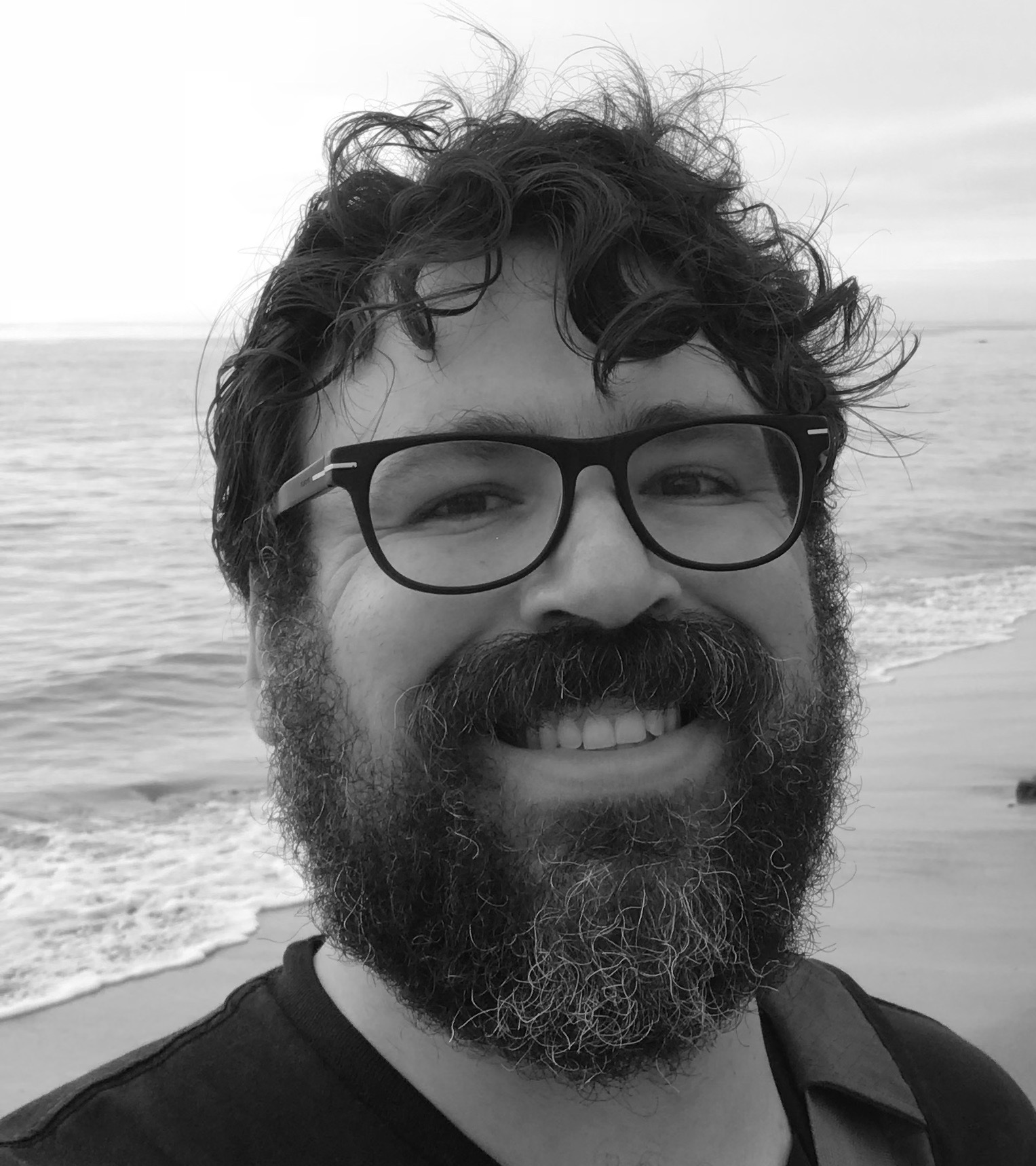}}]{Vince Lyzinski}
 received the BSc degree in Math. from the Univ. of Notre Dame in
2006, the MA in Math. from Johns
Hopkins Univ. (JHU) in 2007, the MSE
in Applied Math. and Stat. from JHU in 2010, and the PhD in Applied Math. and Stat. from JHU in 2013.
From 2013-2014 he was a postdoctoral fellow
with the Applied Math. and Stat.
(AMS) Dept., JHU. During 2014-2017, he
was a senior research scientist with the JHU
HLTCOE and an assistant research professor with the AMS Dept., JHU.
From 2017-2019, he was on the Faculty in the Dept. of Math. and Stat. at the Univ. of Massachusetts, Amherst. Since 2019 he has been on the Faculty in the Dept. of Math. at the Univ. of Maryland, College Park, where he is currently an Assoc. Prof. His research interests include graph matching, statistical
inference on random graphs, pattern recognition, dimensionality reduction, stochastic processes, and high-dim. data analysis.
\end{IEEEbiography}

\onecolumn

\section{Appendix}
Herein, we collect proofs of our main theoretical results as well as additional experiments.

\subsection{Proof of Theorem 2:}
\label{app:pf}
We restate Theorem \ref{thm:main} here before providing a proof.
\vspace{3mm}

\noindent \textbf{Theorem 2:}
Let $A$ and $B$ be two graphs constructed as above. If there is a constant $\alpha\in [1/2,1)$ such that
\begin{itemize}
    \item[i.] $m-k=\Theta(m)$; $m^{1-\alpha}=\omega(\log^4 n)$
    \item[ii.] $\lambda=m^{\alpha}$;
    \item[iii.] $(r_1-r_3)\ll m^{\alpha-1}$; $r_1>r_2>r_3$ are bounded away from $0$ and $1$;
    \item[iv.] $\mu_3>(1-\varepsilon)\mu_1$ and $(1-\varepsilon)\mu_2>\mu_4$; the differences $\mu_3-(1-\varepsilon)\mu_1$, $(1-\varepsilon)\mu_2-\mu_4$, and $\mu_3-\mu_4$ are bounded away from 0;
    \item[v.] $p$ is bounded away from $0$ and $1$;
\end{itemize}
then if $\widetilde \Pi$ is the set of permutations perfectly aligning the weakly embedded template, we have that
$$\mathbb{P}(\text{argmax}_{P\in\Pi}{\tilde f}_{\varepsilon,1}(P,\lambda)\subset \widetilde \Pi)\geq 1-e^{-\omega(\log n)}.$$
{where we recall ${\tilde f}_{\varepsilon,1}(P,\lambda)$ is the objective function defined as ${\tilde f}_{\varepsilon,1}(P,\lambda) = \operatorname{tr}\!\left(\tilde A P \tilde B P^{T}\right) \!+\! \lambda\operatorname{tr}\!\left(S^{(2,\epsilon)}P_{(1)}^T\right)$}

\begin{proof}
{The proof contains three steps. First, we calculate the expected contributions to the differences of the objective functions between (any one of) the permutations that recover the weakly embedded template, $\tilde P$, and another permutation $P$, due to the edge structures and the node features. See Propositions \ref{prop:edgediff} and \ref{prop:nodediff} below. Then, we use our assumptions to demonstrate that the actual difference must concentrate enough around its mean and be bounded away from 0 through the McDiarmid's inequality. See Theorem \ref{thm:McD} below. Finally, by applying a union bound on all possible permutations, we establish that the probability of the difference being bounded away from 0 for all possible permutation matrices is at least $1-e^{-c_n}$ where $c_n$ is proportional to the expected difference calculated in the first step. Combining this with the fact that the expected difference diverges to infinity as $n$ increases, the result is established. See Theorem \ref{thm:unionbdd} below.}
\end{proof}

Let $P^*$ be (any one of) the permutation that maps $A$ to the strongly embedded template in $B$,  $\widetilde P$ be (any one of) the permutation that maps $A$ to the weakly embedded template in $B$, and $P$ be an arbitrary permutation in $\Pi_n$ that does not map $A$ to the weakly embedded template in $B$. Let $C_v = p(1-p)$. Finally let $T_1$ be the set of nodes that $P$ correctly aligns within the strongly embedded template, and $T_2$ the set of nodes that $P$ matches correctly to the weakly embedded template. 
% {\sout{We first calculate the expected value of the objective function, where $\sigma_P$ is the permutation described by the matrix $P$.}}

We first consider the contribution to the objective function of the edge disagreement induced by $P$. {Define $D_E(P):=\operatorname{tr}\left(\tA\widetilde{P}\tB\widetilde{P}^T\right)- \operatorname{tr}\left(\tA P\tB P^T\right)$, which measures the objective function difference due to edge structures between the permutation that recovers the weakly embedded template, $\tilde P$, and any other permutation matrix $P$.}

\begin{prop}
\label{prop:edgediff}
    Denote the counts of the correctly recovered template edges via $P$ and $\tilde P$ with the following table.
    \begin{table}[h!]
    \centering
    \begin{tabular}{c|c|c}
    \backslashbox{P}{$\widetilde P$}& Recovers edges in $ \binom{T_1}{2}\cap \binom{T_2}{2}$ &
    Recovers edges in $\binom{T_2}{2}\setminus \binom{T_1}{2}$\\
    \hline
      Recovers edges in $\binom{T_1}{2}\setminus \binom{T_2}{2}$&  $0$ & $h_1$\\
      \hline
      Recovers edges in $ \binom{T_1}{2}\cap \binom{T_2}{2}$ &  $j_1$      & $0$ \\
      \hline
      Recovers edges in $\binom{T_2}{2}\setminus \binom{T_1}{2}$ & $0$&$h_2$\\
      \hline
      Misaligned template edges & $j_2$ & $h_3$ 
    \end{tabular}
    \end{table}
    
    Then $j_1+j_2 = k(k-1)/2,0\leq h_1,h_2\leq \frac{m(m-1)}{2}-\frac{k(k-1)}{2}=\frac{(m-k)(m+k-1)}{2}$ and $\sum_{i=1}^3 h_i = \frac{(m-k)(m+k-1)}{2}$. {Further,} 

\begin{align}
\E(D_{E}(P)) = 8C_v\left\{j_2r_2+h_1(r_3-r_1)+h_3r_3\right\}
\end{align}
\end{prop}

\begin{proof}
    The equalities and inequalities involving $j$'s and $h$'s are trivial by counting. Also, we have
    \begin{align*}
        \mathbb{E}\left\{\operatorname{tr}\left(\tA P\tB P^T\right)\right\}=&\E\left\{\sum_{i=1}^m \sum_{j=1}^m (2A_{ij}-1)(2
        {B_{\sigma_P(i)\sigma_P(j)}}-1)\right\}\\
        =&\sum_{i=1}^m \sum_{j=1}^m \E\left[4A_{ij}
        {B_{\sigma_P(i)\sigma_P(j)}}
        -2A_{ij}-2{B_{\sigma_P(i)\sigma_P(j)}}+1\right]
    \end{align*}
    Thus, 
    \begin{align*}
   \E\left\{\operatorname{tr}\left(\tA P\tB P^T\right)\right\} =& -8\binom{m}{2}C_v + 4\sum_{i=1}^m\sum_{j=1}^m r_{g_{ij}}C_v+m^2
    \end{align*}
    where{, recall that we use $\binom{S}{2}$ to denote the collection of all 2-element subsets of the set $S$,}
    $$\displaystyle g_{ij} = \begin{cases}
        1, &\text{if }\{i,j\}\in \binom{T_1}{2}\setminus \binom{T_2}{2};\\
        2, &\text{if }\{i,j\}\in \binom{T_1}{2}\cap \binom{T_2}{2};\\
        3, &\text{if }\{i,j\}\in \binom{T_2}{2}\setminus \binom{T_1}{2};\\
        0, &\text{otherwise}.
    \end{cases}$$
    
    {Note that the constant terms $-8\binom{m}{2}C_v$ and $m^2$ cancel when taking differences. Also, $j_1$ and $h_2$ terms vanish since they correspond to cases where $r_{g_{ij}}$ values agree for $P$ and $\tilde P$, so we know}
    \begin{align*}
        \E(D_{E}(P))&=\E\left\{\operatorname{tr}\left(\tA\widetilde{P}\tB\widetilde{P}^T\right)- \operatorname{tr}\left(\tA P\tB P^T\right)\right\}\\
        &= 8C_v\left\{j_2r_2+h_1(r_3-r_1)+h_3r_3\right\}
        \end{align*}
\end{proof}

Now, we consider $S$ of the form
$$S=\bordermatrix{
    &m-k&k&m-k&n-2m+k\cr 
    m-k& S^{11}&S^{12}&S^{13}&S^{14}\cr
    k&S^{21}& S^{22}&S^{23}&S^{24}}$$
    where all entries of $S$ are independent---bounded in $[0,1]$---random variables, and where 
    \begin{align*}
    &\text{ the diagonal elements of }S^{11}\text{ have mean }\mu_1\\
        &\text{ the diagonal elements of }S^{22}\text{ have mean }\mu_2\\
        &\text{ the diagonal elements of }S^{13}\text{ have mean }\mu_3
    \end{align*}
    and all other entries have mean $\mu_4$. 
    
If the strongly embedded template is penalized, then $S$ is weighted via (where ``$\circ$'' is the matrix Hadamard product)
$${S^{(2, \epsilon)}}=\bordermatrix{
    &m-k&k&m-k&n-2m+k\cr 
    m-k& (\mathbf{1}_{m-k,m-k}-\epsilon I_{m-k})\circ S^{11}&S^{12}&S^{13}&S^{14}\cr
    k&S^{21}& (\mathbf{1}_{k,k}-\epsilon I_{k})\circ S^{22}&S^{23}&S^{24}}$$
    
    We next consider the contribution to the objective function of the features/similarity induced by $P$. {Like before, we define $D_F(P):= \operatorname{tr}\left(S^{(2,\epsilon)}\tilde{P}^T\right)- \operatorname{tr}\left(S^{(2,\epsilon)}P^T\right)$ which measures the objective function difference due to node similarities between the permutation that recovers the weakly embedded template, $\tilde P$, and any other permutation matrix $P$.}

\begin{prop}
\label{prop:nodediff}
    Denote the counts of the recovered template nodes via $P$ and $\tilde P$ with the following table
\begin{table}[h!]
    \centering
    \begin{tabular}{c|c|c}
    \backslashbox{P}{$\tilde{P}$}& \text{Recovers $T_1\cap T_2$} & \text{Recovers $T_2\setminus T_1$}\\
    \hline
    \text{Recovers $T_1\setminus T_2$} & $0$ & $b_1$\\
    \hline
    \text{Recovers $T_1\cap T_2$} & $a_1$&$0$\\
    \hline
    \text{Recovers $T_2\setminus T_1$} & $0$&$b_2$\\
    \hline
    \text{Misaligned template node}&$a_2$&$b_3$
    \end{tabular}
\end{table}
    
    Then $0\leq a_1\leq k, a_1+a_2=k, 0\leq b_1,b_2<m-k$ and $\sum_{i=1}^3b_i=m-k$. Also,
    
    \begin{align*}
\binom{a_1}{2}&=j_1,\quad
\binom{a_2}{2}+a_1a_2=j_2,\quad\binom{b_2}{2}+b_2a_1=h_2,\\  
\binom{b_1}{2}+b_1a_1&=h_1,\quad
\binom{b_3}{2}+b_1b_2+b_1b_3+b_2b_3+a_2(m-k)+b_3a_1=h_3.
\end{align*}
where $j$'s and $h$'s are the counts from the table of Proposition \ref{prop:edgediff} above. Moreover, 

\begin{align}
\E(D_{F}(P))= a_2[(1-\varepsilon)\mu_2-\mu_4]+b_1[\mu_3-(1-\varepsilon)\mu_1]+b_3[\mu_3-\mu_4].
\end{align}
\end{prop}

\begin{proof}
    The equalities and inequalities involving $a$'s and $b$'s, as well as the relationship between $a$'s $b$'s and $j$'s, $h$'s are trivial by counting.\\
    Now, we have
    $$\begin{aligned}
\E\left[\operatorname{tr}\left(S^{(2,\epsilon)}P^T\right)\right]=&\E\left\{\sum_{i=1}^m  (1-\varepsilon)^{\mathbbm{1}\{\sigma_P(i) =\sigma_{P^*}(i)\}}S_{i,\sigma_P(i)}\right\}\\
        =&\sum_{i=1}^m (1-\varepsilon)^{\mathbbm{1}\{\sigma_P(i) =\sigma_{P^*}(i)\}}\E(S_{i,\sigma_P(i)})\\
        =& \sum_{i=1}^m (1-\varepsilon)^{\mathbbm{1}\{\sigma_P(i) =\sigma_{P^*}(i)\}} \mu_{w_i}
    \end{aligned}$$
    where $$\displaystyle w_{i} = \begin{cases}
        1, &\text{if }i\in T_1\setminus T_2;\\
        2, &\text{if }i\in T_1\cap T_2;\\
        3, &\text{if }i\in T_2\setminus T_1;\\
        4, &\text{otherwise}.
    \end{cases}$$

    $a_1$ and $b_2$ terms vanish when taking differences since they correspond to cases where $\mathbbm{1}\{\sigma_P(i) =\sigma_{P^*}(i)\}$ and $\mu_{w_{i}}$ values agree for $P$ and $\tilde P$, so we know
\begin{align*}
\E(D_{F}(P))&=\E\left\{\operatorname{tr}\left(S^{(2,\epsilon)}\tilde{P}^T\right)- \operatorname{tr}\left(S^{(2,\epsilon)}P^T\right)\right\}\\
&= a_2[(1-\varepsilon)\mu_2-\mu_4]+b_1[\mu_3-(1-\varepsilon)\mu_1]+b_3[\mu_3-\mu_4].
\end{align*}
\end{proof}

Let {$g(P, \tilde{P}) = \tilde {f}_{\varepsilon,1}(\tilde{P},\lambda)-\tilde {f}_{\varepsilon,1}({P},\lambda)$} so that 
\begin{align*}
    \E g(P, \widetilde{P})&=\E(D_E(P)+\lambda D_F(P))\\
    &=8C_v\Big\{\left(\binom{a_2}{2}+a_1a_2\right)r_2+\left(\binom{b_1}{2}+b_1a_1\right)(r_3-r_1)\\
    &\hspace{5mm}+\left(\binom{b_3}{2}+b_1b_2+b_1b_3+b_2b_3+a_2(m-k)+b_3a_1\right)r_3\Big\}\\
    &\hspace{5mm}+\lambda(a_2[(1-\varepsilon)\mu_2-\mu_4]+b_1[\mu_3-(1-\varepsilon)\mu_1]+b_3[\mu_3-\mu_4])\\
    &\geq 8C_v\Big\{\underbrace{\left(\binom{b_1+a_2+b_3}{2}+(b_1+a_2+b_3)(b_2+a_1)\right)}_{:=n_E} r_3-\left(\binom{b_1}{2}+b_1a_1\right)r_1\Big\}\\
&\hspace{5mm}+\lambda(\underbrace{a_2+b_1+b_3}_{:=n_F})*\underbrace{\min([(1-\varepsilon)\mu_2-\mu_4], [\mu_3-(1-\varepsilon)\mu_1], [\mu_3-\mu_4])}_{:=c_\varepsilon}\\
&=8C_v\left(n_E r_3-\left(\binom{b_1}{2}+b_1a_1\right)r_1\right)+\lambda n_Fc_{\varepsilon}.
\end{align*}
%________________________________________________---------------------------------------------------
Considering the case where $b_1=m-k$ and $a_1=k$, we see that for the above expectation to be diverging to infinity, it suffices that
\begin{equation}
\label{eq:gr}
\lambda c_\varepsilon\gg (r_1-r_3)(m{+}k)C_v;
\end{equation}
Equation \ref{eq:gr} holds true under assumptions ii and iii of Theorem \ref{thm:main}.

Next, we state the McDiarmid's inequality and use it to show for any $P\in\Pi_n$, $g(P, \tilde P)$ concentrates around its expectation.

\begin{thm}[McDiarmid's inequality \cite{mcdiarmid1989method}]
\label{thm:McD}
    Let $X_1, \ldots, X_n$ be independent random variables, where $X_i\in \mathcal{X}_i$. Let $f: \mathcal{X}_1 \times \cdots \times \mathcal{X}_n \rightarrow \mathbb{R}$ be any function that satisfy: there exists $c_1, \ldots c_n$ such that for every $i\in [n]$ and every $\left(x_1, \ldots, x_n\right) \in \mathcal{X}_1 \times \ldots \times \mathcal{X}_n$, we have
    $$\sup _{x_i^{\prime} \in \mathcal{X}_i}\left|f\left(x_1, \ldots, x_{i-1}, x_i, x_{i+1}, \ldots, x_n\right)-f\left(x_1, \ldots, x_{i-1}, x_i^{\prime}, x_{i+1}, \ldots, x_n\right)\right| \leq c_i$$
    Then for any $t>0$, we have
    $$\Prb\left(f\left(X_1, \ldots, X_n\right)-\mathbb{E}\left[f\left(X_1, \ldots, X_n\right)\right] \geq t\right) \leq \exp \left(-\frac{2 t^2}{\sum_{i=1}^n c_i^2}\right)$$
\end{thm}

From the forms of $D_E(P)$ and $D_F(P)$, we have that $g(P, \widetilde{P})$ is a function of at most (as $P$ and $\widetilde P$ agree on $b_2+a_1$ template vertices and disagree on the rest)
\begin{equation}
\label{eq:varct}
    \underbrace{2m-2b_2-2a_1}_{\text{from }D_F(P)}+\underbrace{\binom{m-b_2-a_1}{2}+(m-b_2-a_1)(b_2+a_1)}_{\text{from }D_E(P)}
\end{equation}
random variables, and changing any of the variables from $D_F(P)$ can change $g(P, \widetilde{P})$ by at most $4\lambda$, and changing any of the variables from $D_E(P)$ can change $g(P, \widetilde{P})$ by at most a bounded constant (bounded above by 8 for example).
Lastly, note that $m-b_2-a_1=b_1+b_3+a_2=n_F$ and 
$$\binom{m-b_2-a_1}{2}+(m-b_2-a_1)(b_2+a_1)=
\binom{b_1+a_2+b_3}{2}+(b_1+a_2+b_3)(b_2+a_1)=n_E.
$$
So
$$
n_E=n_F\left(  \frac{n_F-1}{2}+b_2+a_1\right)
$$

Assumption iv. of Theorem \ref{thm:main} gives that $(1-\varepsilon)\mu_2>\mu_4$ and $(1-\varepsilon)\mu_1<\mu_3$, and that the differences 
$$
(1-\varepsilon)\mu_2-\mu_4,\quad \mu_3-\mu_4, \text{ and }\mu_3-(1-\varepsilon)\mu_1
$$ are bounded away from $0$. Further, Eq.\@ \eqref{eq:gr} implies that 
\begin{align*}
\lambda n_Fc_{\varepsilon}&\gg n_F (r_1-r_3)(m+k)C_v\\
&\gtrsim \binom{b_1}{2}(r_1-r_3)C_v\\
&\gtrsim C_v\left( \left(\binom{b_1}{2}+b_1a_1\right)r_1-r_3 n_E\right)
\end{align*}
and $\E g(P, \widetilde{P})=\Omega(\lambda n_Fc_{\varepsilon})$.

{Apply theorem \ref{thm:McD} with $g(P, \tilde{P})$ as a function of $n_E+2n_F$ random variables from Eq. \ref{eq:varct} and $c_i=8$ for random variables from $D_E(P)$ and $c_i=4\lambda$ for random variables from $D_F(P)$, we see} that for $n$---and hence $m=m_n$---sufficiently large, (where $\xi>0$ is a constant that can change line--to--line)
\begin{align}
\mathbb{P}&({g(P, \widetilde{P})\leq 0})\leq \mathbb{P}(|g(P, \widetilde{P})-\mathbb{E}[g(P, \widetilde{P})]|\geq \mathbb{E}[g(P, \widetilde{P})])\notag\\
&\leq 2 \exp\left\{-\xi\frac{
C_v^2\left(n_E r_3-\left(\binom{b_1}{2}+b_1a_1\right)r_1\right)^2+\lambda^2 n_F^2c_{\varepsilon}^2}{n_E+\lambda^2 n_F}\right\}\notag\\
\label{eq:bnd}
&\leq 2 \exp\left\{\!-\xi\frac{
C_v^2\!\left(\left(b_3^2+a_2^2\!+\!(b_1\!+\!b_3\!+\!a_2)(b_2+a_1)\right) r_3\!-\!b_1^2(r_1\!-\!r_3)\right)^2\!+\!m^{2\alpha} n_F^2c_{\varepsilon}^2}{m^{2\alpha} n_F}\right\}
\end{align}

Next, we have the following lemma

\begin{lem}
\label{lem:binaryentropy}
    Let $H(\cdot)$ be the {binary} entropy function as defined in \cite{cover1999elements}. Partition $\Pi_n$ based on $\Pi_n\ni P\sim Q\in \Pi_n$ if the first $m$ rows of $P$ and $Q$ are exactly equal in to equivalence classes. Then within each partition, our objective function $\tilde{f}_{\varepsilon,1}(\cdot,\lambda)$ takes the same value; moreover, the order of each equivalence class is bounded above by $2^{k H\left(\frac{a_1}{k}\right)+(m-k)H\left(\frac{b_1}{m-k}\right)+(m-k)H\left(\frac{b_2}{m-k}\right)}n^{a_2+b_3}$.
\end{lem}
\begin{proof}
    Fix an equivalence class $C$. For any $P, Q\in C$, the contribution due to edge terms are equal since the bottom $n-m$ rows of both $P$ and $Q$ contribute 0 to the value of the objective function due to padding. Also, the contribution due to feature terms is independent of the bottom $n-m$ rows of the permutation matrix. Thus $\tilde{f}_{\varepsilon,1}(P,\lambda)=\tilde{f}_{\varepsilon,1}(Q,\lambda)$ which proves the objective function takes the same value in $C$. \\
    Now, by (7.14) from \cite{cover1999elements}
    \begin{equation}
\label{eqn:BiEnBd}
    |C|=\binom{k}{a_1}\binom{m-k}{b_1}\binom{m-k}{b_2}n^{a_2+b_3}\leq 2^{k H\left(\frac{a_1}{k}\right)+(m-k)H\left(\frac{b_1}{m-k}\right)+(m-k)H\left(\frac{b_2}{m-k}\right)}n^{a_2+b_3}.
\end{equation}
\end{proof}

Finally, we apply the union bound on all permutation matrices and get
\begin{thm}
\label{thm:unionbdd}
    With the assumptions of Theorem \ref{thm:main}, by Lemma \ref{lem:binaryentropy} and Equation \ref{eq:bnd}, we have
    \begin{equation}
        \mathbb{P}(\exists\, P\neq \widetilde P\text { s.t. }{g(P, \widetilde{P})\leq 0})\leq2e^{-\omega(\log n)}
    \end{equation}
\end{thm}

\begin{proof}
Apply a union bound over all $P$ with the same counts ($a$'s and $b$'s), modulo the equivalence between permutations with the same first $m$ rows; further apply a union bound over all possible counts of the $a$'s and $b$'s, we have (where $\xi>0$ is a constant that can change line--to--line)
\begin{align}
&\mathbb{P}(\exists\, P\neq \widetilde P\text { s.t. }{g(P, \widetilde{P})\leq 0})\leq\notag\\ &\sum_{a_1=0}^k\sum_{b_1=0}^{m-k}\sum_{b_2=0}^{m-k-b_1}
2\exp\Bigg\{\!-\xi\frac{
C_v^2\!\left(\left(b_3^2+a_2^2\!+\!(b_1\!+\!b_3\!+\!a_2)(b_2+a_1)\right) r_3\!-\!b_1^2(r_1\!-\!r_3)\right)^2\!+\!m^{2\alpha} n_F^2c_{\varepsilon}^2}{m^{2\alpha} n_F}\notag\\
&\hspace{10mm}+k H\left(\frac{a_1}{k}\right)\log(2)
 +(m-k)H\left(\frac{b_1}{m-k}\right)\log(2)+(m-k)H\left(\frac{b_2}{m-k}\right)\log(2)\notag\\
&\hspace{10mm}+(a_2+b_3)\log(n)\Bigg\}\label{eq:union}
\end{align}
To tackle the above Eq. \eqref{eq:union}, we consider several cases:
\begin{itemize}
\item If
$b_3^2=\Omega(m^{3/2+\alpha/2})$ or
$a_2^2=\Omega(m^{3/2+\alpha/2})$, then the exponential in Eq.\@ \eqref{eq:union} can be bounded above by (where $\xi>0$ is a constant that can change line--to--line)
\begin{align}
\label{eq:bnd3}
2 \exp\left\{-\xi\frac{
C_v^2m^{3+\alpha}}{m^{2\alpha} n_F}+m\log n\right\}\leq 2 \exp\Big\{-\xi
C_v^2m(\underbrace{m^{2-\alpha}-\log n}_{:=\Theta(m^{2-\alpha})})\Big\}=2e^{-\omega(\log n)}
\end{align}
\item We next consider the case where both 
$b_3^2=o(m^{3/2+\alpha/2})$ and
$a_2^2=o(m^{3/2+\alpha/2})$, in this case either $b_1=\Theta(m)$, $b_2=\Theta(m)$, or both are $\Theta(m)$.
\begin{itemize}
\item $b_1=\Theta(m)$ and $b_2=\Theta(m)$:  Then the $b_1*b_2$ term in the exponent in Eq.\@ \eqref{eq:union} yields the upper bound (where $\xi>0$ is a constant that can change line--to--line)
\begin{align}
\label{eq:bnd4}
2 \exp\left\{-\xi\frac{
C_v^2m^{4}}{m^{2\alpha} n_F}+m\log n\right\}\leq 2 \exp\Big\{-\xi
C_v^2m(\underbrace{m^{3-2\alpha}-\log n}_{:=\Theta(m^{3-2\alpha})})\Big\}=2e^{-\omega(\log n)}
\end{align}

\item $b_2=o(m)$: In this case, $b_1(1+o(1))=m-k$.
Note that if $x(1+o(1))=y$, then $x/y(1+o(1))=1$.
If $x\leq y$, then $1-x/y=o(1)$.
As $\lim_{z\rightarrow 0}-z\log_2 z=0$, we then have
\begin{align*}
yH(x/y)&=y\left[\frac{x}{y}\underbrace{\left[-\log_2 \left(\frac{x}{y}\right)\right]}_{=o(1)}+\underbrace{\left[-\left(1-\frac{x}{y}\right)\log_2\left(1-\frac{x}{y}\right)\right]}_{=o(1)}  \right]\\
&=o(y)
\end{align*}
Similarly, if $y=\omega(1)$ and $x=o(y)$, then \begin{align*}
yH(x/y)&=y\Bigg[\underbrace{-\frac{x}{y}\log_2 \left(\frac{x}{y}\right)}_{=o(1)}+\left(1-\frac{x}{y}\right)\underbrace{\left[-\log_2\left(1-\frac{x}{y}\right)\right]}_{=o(1)}  \Bigg]\\
&=o(y)
\end{align*}
The $n_F^2*m^{2\alpha}$ term in the exponent in Eq.\@ \eqref{eq:union} yields the upper bound (where $\xi>0$ is a constant that can change line--to--line; note by assumption $c_\varepsilon$ is bounded away from $0$)
\begin{align}
\label{eq:bnd5}
&2 \exp\Big\{-\xi\frac{
n_F^2*m^{2\alpha}c_\varepsilon^2}{m^{2\alpha} n_F}+k H\left(\frac{a_1}{k}\right)\log(2)
 +(m-k)H\left(\frac{b_1}{m-k}\right)\log(2)\notag\\
&\hspace{10mm}+(m-k)H\left(\frac{b_2}{m-k}\right)\log(2)+(a_2+b_3)\log(n)\Bigg\}\notag\\
&\leq 2 \exp\Big\{-\xi
mc_\varepsilon^2+o(m)+o(m^{3/4+\alpha/4}\log(n))\Bigg\}=2e^{-\omega(\log n)}
\end{align}
%%%%%%%%%%%%%%%%%%%%%%%%%%%%%%%%%
\item $b_1=o(m)$: In this case, $b_2(1+o(1))=(m-k)$.
We make use here of the alternate bound where if $y=\omega(1)$ and $x=o(y)$, then (as $
\lim_{z\rightarrow 0}-\log_2(1-z)/z=1/\log(2)$)
\begin{align*}
yH(x/y)&=y\Bigg[-\frac{x}{y}\log_2 \left(\frac{x}{y}\right)-\left(1-\frac{x}{y}\right)\log_2\left(1-\frac{x}{y}\right)  \Bigg]\\
&=O(x\log_2(y))+O\left(y\frac{x}{y}\left(1-\frac{x}{y}\right)\right)=O(x\log(y)).
\end{align*}
We also note that here
\begin{align*}
(&m-k)H\left(\frac{b_2}{m-k}\right)\\
&=(m-k)H\left(1-\frac{b_1+b_3}{m-k}\right)\\
&=(m-k)\left[-\left(1-\frac{b_1+b_3}{m-k}\right)\log_2\left(1-\frac{b_1+b_3}{m-k}\right)-\left(\frac{b_1+b_3}{m-k}\right)\log_2\left(\frac{b_1+b_3}{m-k}\right)
\right]\\
&=O\left((m-k)\frac{b_1+b_3}{m-k}\left(1-\frac{b_1+b_3}{m-k}\right)\right)+O((b_1+b_3)\log_2(m-k))\\
&=O((b_1+b_3)\log(m-k))
\end{align*}
The $n_F^2*b_2^2$ term in the exponent in Eq.\@ \eqref{eq:union} yields the upper bound (where $\xi>0$ is a constant that can change line--to--line, and we use $k H\left(\frac{a_1}{k}\right)=k H\left(1-\frac{a_2}{k}\right)$ as above)
\begin{align}
\label{eq:bnd6}
&2 \exp\Big\{-\xi
n_Fm^{2-2\alpha}+k H\left(\frac{a_1}{k}\right)\log(2)
 +(m-k)H\left(\frac{b_1}{m-k}\right)\log(2)\notag\\
&\hspace{10mm}+(m-k)H\left(\frac{b_2}{m-k}\right)\log(2)+(a_2+b_3)\log(n)\Bigg\}\notag\\
&\leq 2 \exp\Big\{-\xi
n_Fm^{2-2\alpha}+O((b_1+b_3+a_2)\log(m))+n_F\log(n)\Bigg\}=2e^{-\omega(\log n)}
\end{align}
\end{itemize}
\end{itemize}
Therefore, by Equations \ref{eq:bnd3}, \ref{eq:bnd4}, \ref{eq:bnd5} and \ref{eq:bnd6}, we have that individual summands of the right-hand-side of Eq.\@ \eqref{eq:union} is bound above by $2e^{-\omega(\log n)}$, and thus
\begin{align*}
&\mathbb{P}(\exists\, P\neq \widetilde P\text { s.t. }{g(P, \widetilde{P})\leq 0})\leq 2e^{-\omega(\log n)} \left(\sum_{a_1=0}^k\sum_{b_1=0}^{m-k}\sum_{b_2=0}^{m-k-b_1} 1\right) \leq 2e^{-\omega(\log n)+\ln k+2\ln m}=2e^{-\omega(\log n)}
\end{align*}
as desired.
\end{proof}

\subsection{More Experiments}

\subsubsection{Additional two overlapping templates experiments}
\label{app:N2}
We plot the cases of $k=15$ and $k=40$ using the seeded \mg algorithm with 5 seeds randomly selected from the overlapping nodes of $B^{(1)}$ and $B^{(2)}$, for the same parameters as described in \ref{sec:N2}.

\begin{figure}[t!]
    \centering
    \includegraphics[width = 0.8\textwidth]{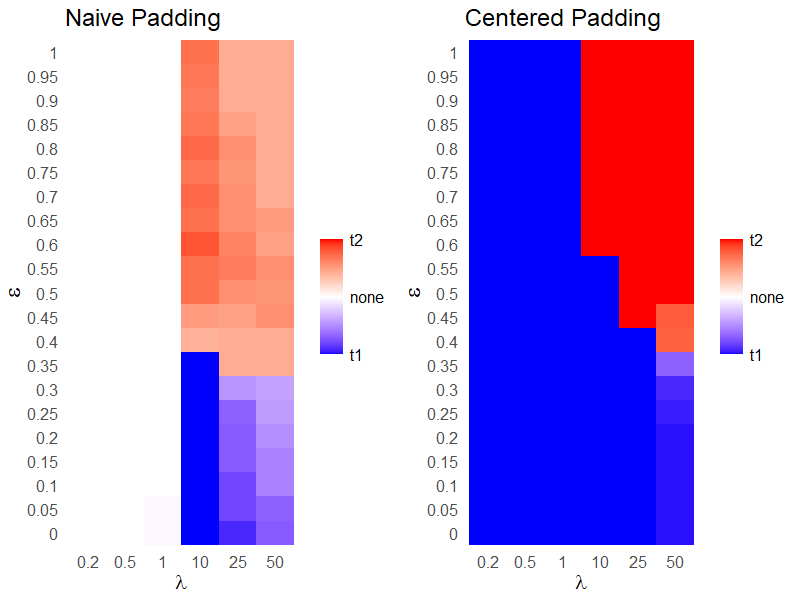}
    \caption{We fix $k=15$ and use the seeded \mg  algorithm to match $A$ with $B$ using 5 seeds randomly selected from the overlapping nodes of $B^{(1)}$ and $B^{(2)}$ as described in Section \ref{sec:N2}.
    We plot the recovering results over $\varepsilon$ (here $\varepsilon$ is used to penalize the stronger of the two embedded templates) and $\lambda$, averaged by 20 Monte-Carlo simulations, where blue means the recovered template is closer to $B^{(1)}$ (the stronger embedded template), red means the recovered template is closer to $B^{(2)}$ (the weaker embedded template), and white means there is a tie in the 20 simulations or the recovered template is not close to either $B^{(1)}$ or $B^{(2)}$.}
\end{figure}

\begin{figure}[t!]
    \centering
    \includegraphics[width = 0.8\textwidth]{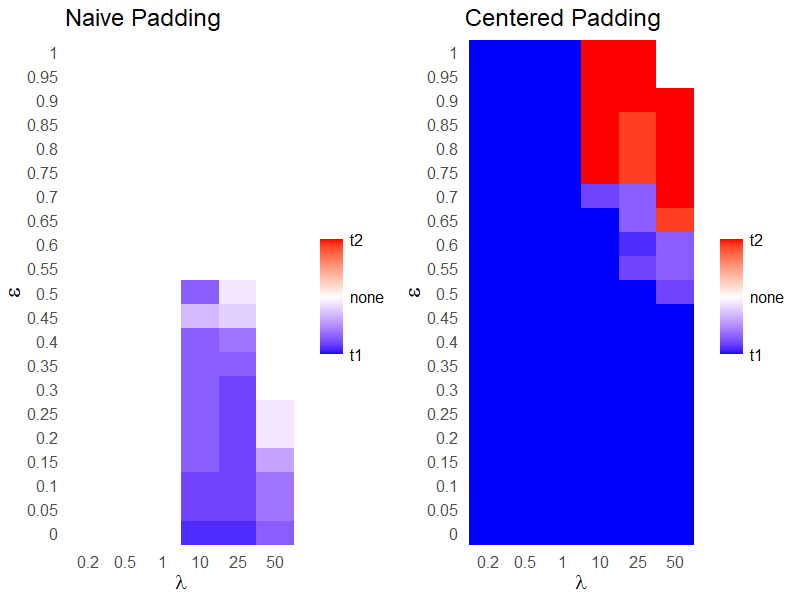}
    \label{fig:N2_k40}
    \caption{We fix $k=40$ and use the seeded \mg  algorithm to match $A$ with $B$ using 5 seeds randomly selected from the overlapping nodes of $B^{(1)}$ and $B^{(2)}$ as described in Section \ref{sec:N2}.
    We plot the recovering results over $\varepsilon$ (here $\varepsilon$ is used to penalize the stronger of the two embedded templates) and $\lambda$, averaged by 20 Monte-Carlo simulations, where blue means the recovered template is closer to $B^{(1)}$ (the stronger embedded template), red means the recovered template is closer to $B^{(2)}$ (the weaker embedded template), and white means there is a tie in the 20 simulations or the recovered template is not close to either $B^{(1)}$ or $B^{(2)}$. Note that the naive padding never recovered anything closer to $B^{(2)}$.}
    \label{fig:bigK}
\end{figure}

\subsubsection{Additional three overlapping templates experiments}
\label{app:N3}
We first plot the simulated results where the parameters correspond to Figure \ref{fig:N3} of \ref{sec:N3} but with the naive padding. 

\begin{figure}[t!]
    \centering
    \includegraphics[width = 0.7\textwidth]{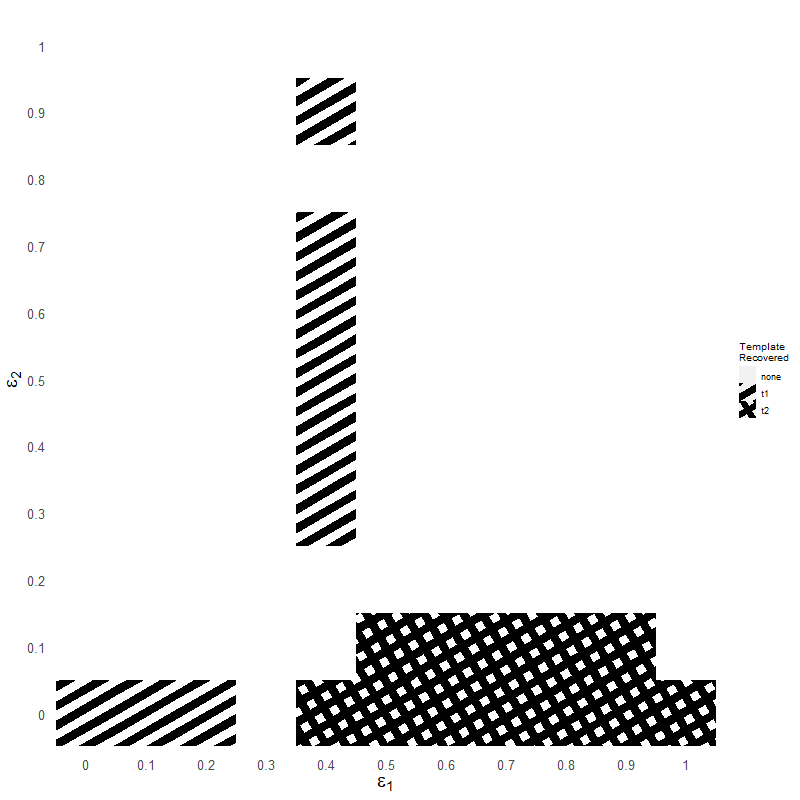}
    \caption{We fix $k=10,\lambda=25$ and use the seeded \mg algorithm with the naive padding to match $A$ with $B$ using 5 seeds randomly selected from the overlapping nodes of $B^{(1)}, B^{(2)}$ and $B^{(3)}$, where $B^{(1)}, B^{(2)}$ and $B^{(3)}$ are induced subgraph of $B$ such that graphs $A$ and $B$ follows multiple correlated ER model as described in Section \ref{sec:N3}. We plot the recovering results over $\varepsilon_1$ (penalty applied to the diagonal elements of $S^{(11)}, S^{(22)}$) and $\varepsilon_2$ (penalty applied to the diagonal elements of $S^{(13)}, S^{(22)}$), averaged by 20 Monte-Carlo simulations. 
In the figure, the different patterns represent which template was recovered (in majority): t1 for template 1, t2 for template 2, and we fail to recover template 3, with white squares corresponding to the case when none of the three templates was recovered.}
\end{figure}

Next, with the same correlation parameters $\{r_j\}_{j=1}^4$ and feature mean parameters $\{\mu_j\}_{j=1}^5$, we plot the simulated results for $k=10, \lambda=10$ and $k=40,\lambda=25$ using the centered padding. 

\begin{figure}[t!]
    \centering
    \includegraphics[width = 0.7\textwidth]{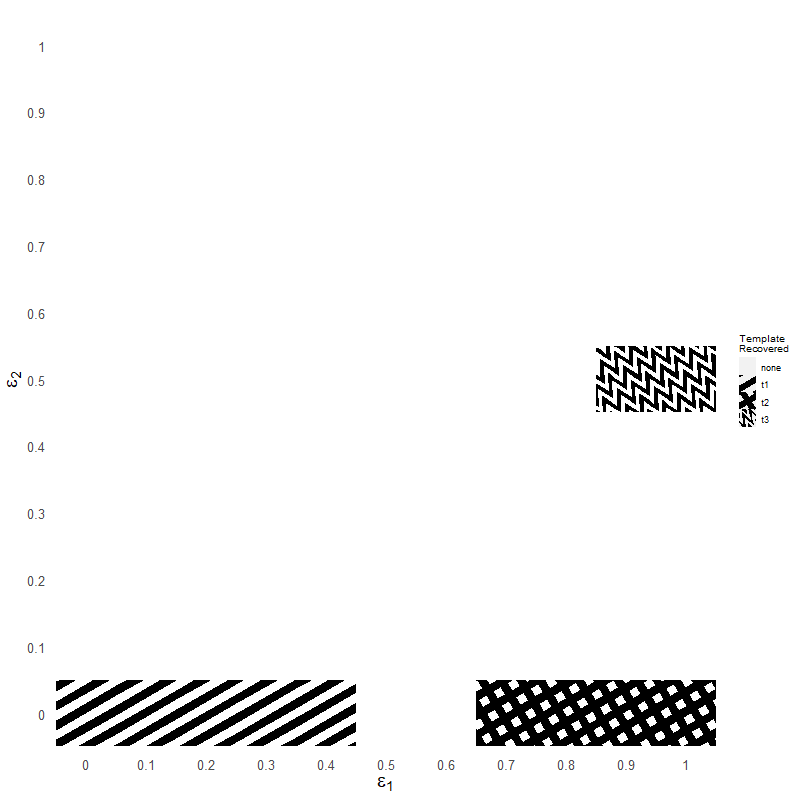}
    \caption{We fix $k=10,\lambda=10$ and use the seeded \mg algorithm with the centered padding to match $A$ with $B$ using 5 seeds randomly selected from the overlapping nodes of $B^{(1)}, B^{(2)}$ and $B^{(3)}$, where $B^{(1)}, B^{(2)}$ and $B^{(3)}$ are induced subgraph of $B$ such that graphs $A$ and $B$ follows multiple correlated ER model as described in Section \ref{sec:N3}. We plot the recovering results over $\varepsilon_1$ (penalty applied to the diagonal elements of $S^{(11)}, S^{(22)}$) and $\varepsilon_2$ (penalty applied to the diagonal elements of $S^{(13)}, S^{(22)}$), averaged by 20 Monte-Carlo simulations. 
In the figure, the different patterns represent which template was recovered (in majority): t1 for template 1, t2 for template 2, and t3 for template 3, with white squares corresponding to the case when none of the three templates was recovered.}
\end{figure}

\begin{figure}[t!]
    \centering
    \includegraphics[width = 0.7\textwidth]{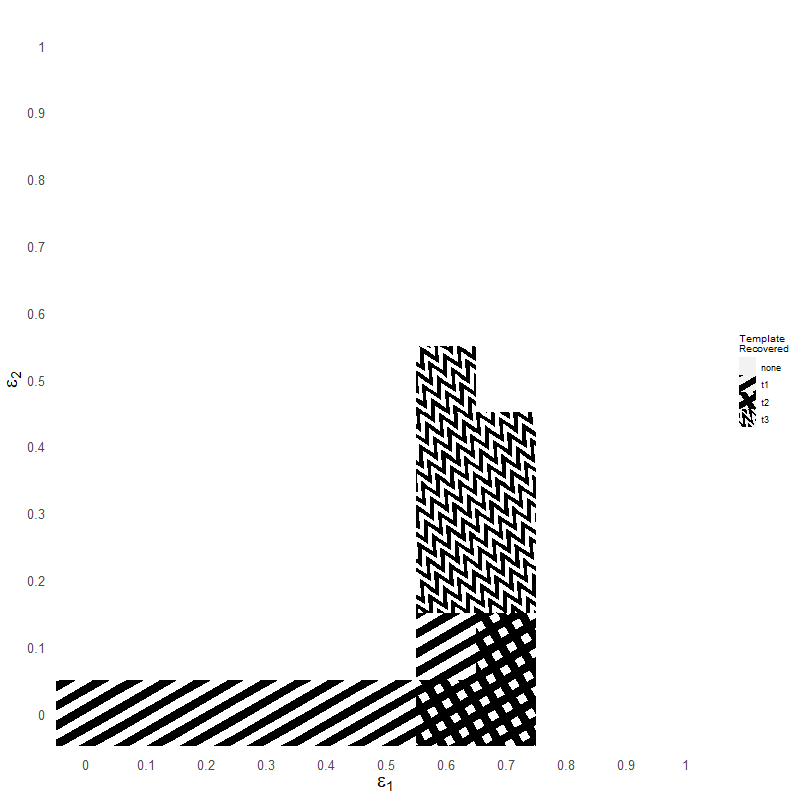}
    \caption{We fix $k=40,\lambda=25$ and use the seeded \mg algorithm with the centered padding to match $A$ with $B$ using 5 seeds randomly selected from the overlapping nodes of $B^{(1)}, B^{(2)}$ and $B^{(3)}$, where $B^{(1)}, B^{(2)}$ and $B^{(3)}$ are induced subgraph of $B$ such that graphs $A$ and $B$ follows multiple correlated ER model as described in Section \ref{sec:N3}. We plot the recovering results over $\varepsilon_1$ (penalty applied to the diagonal elements of $S^{(11)}, S^{(22)}$) and $\varepsilon_2$ (penalty applied to the diagonal elements of $S^{(13)}, S^{(22)}$), averaged by 20 Monte-Carlo simulations. 
In the figure, the different patterns represent which template was recovered (in majority): t1 for template 1, t2 for template 2, and t3 for template 3, with white squares corresponding to the case when none of the three templates was recovered.}
\end{figure}
\clearpage

\subsubsection{Additional Brain MRI plots}
\label{app:brain}
\begin{figure}[b!]
    \centering
    \includegraphics[width=0.7\textwidth]{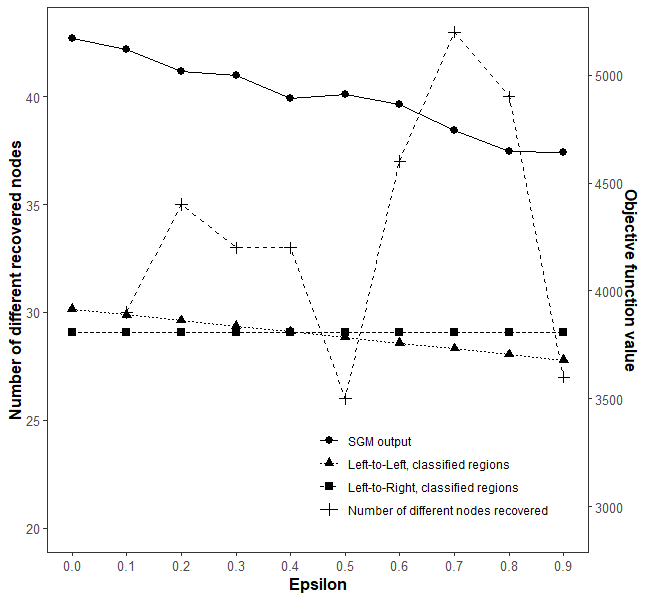}
    \caption{We run our proposed algorithm using the seeded \mg algorithm with 500 restarts and 5 seeds selected from the node pairs $\{(s_j,w_j)\}_{j=1}^6$ as described in \ref{sec:brain}, taking the result with highest objective function value (Eq. \ref{npgmp}, $\lambda=1$) as the output. For each $\varepsilon$, we compute the GM objective function value (left axis) of the resulting matrix with the template; we also computed the objective function value with respect to the alignment given by the template to the same classified brain region in the left hemisphere in $B$ (Left--to--Left in the plot), as well as the objective function value given by the template to the symmetric region from the right hemisphere in $B$ (Left--to--right in the plot). Also for $\varepsilon>0$, we calculated the number of novel nodes recovered in each matching compared to the subgraph detected with $\varepsilon=0$ (right axis).}
    \label{fig:MRI_App}
\end{figure}
\clearpage

\subsubsection{Additional TKB templates}
\label{app:tkb}
\begin{figure}[b!]
    \centering
    \includegraphics[width=1\textwidth]{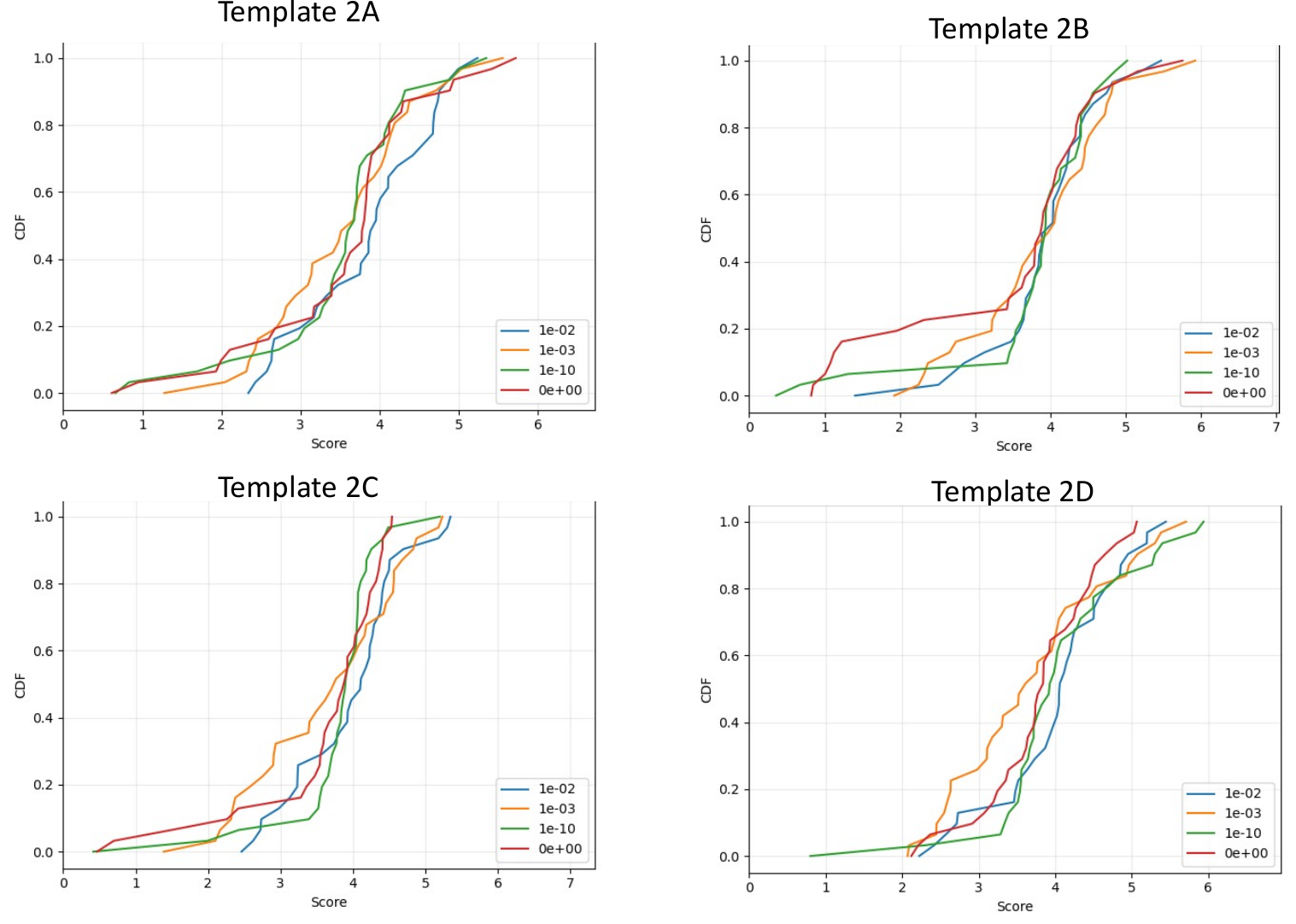}
    \caption{We run 32 random restarts of the \mg algorithm for each template recovery, plotting the empirical CDF of the GED of the recovered templates.
Different penalization values are represented with different colors in the plot.}
    \label{fig:tkb2}
\end{figure}

\begin{figure}[b!]
    \centering
    \includegraphics[width=1\textwidth]{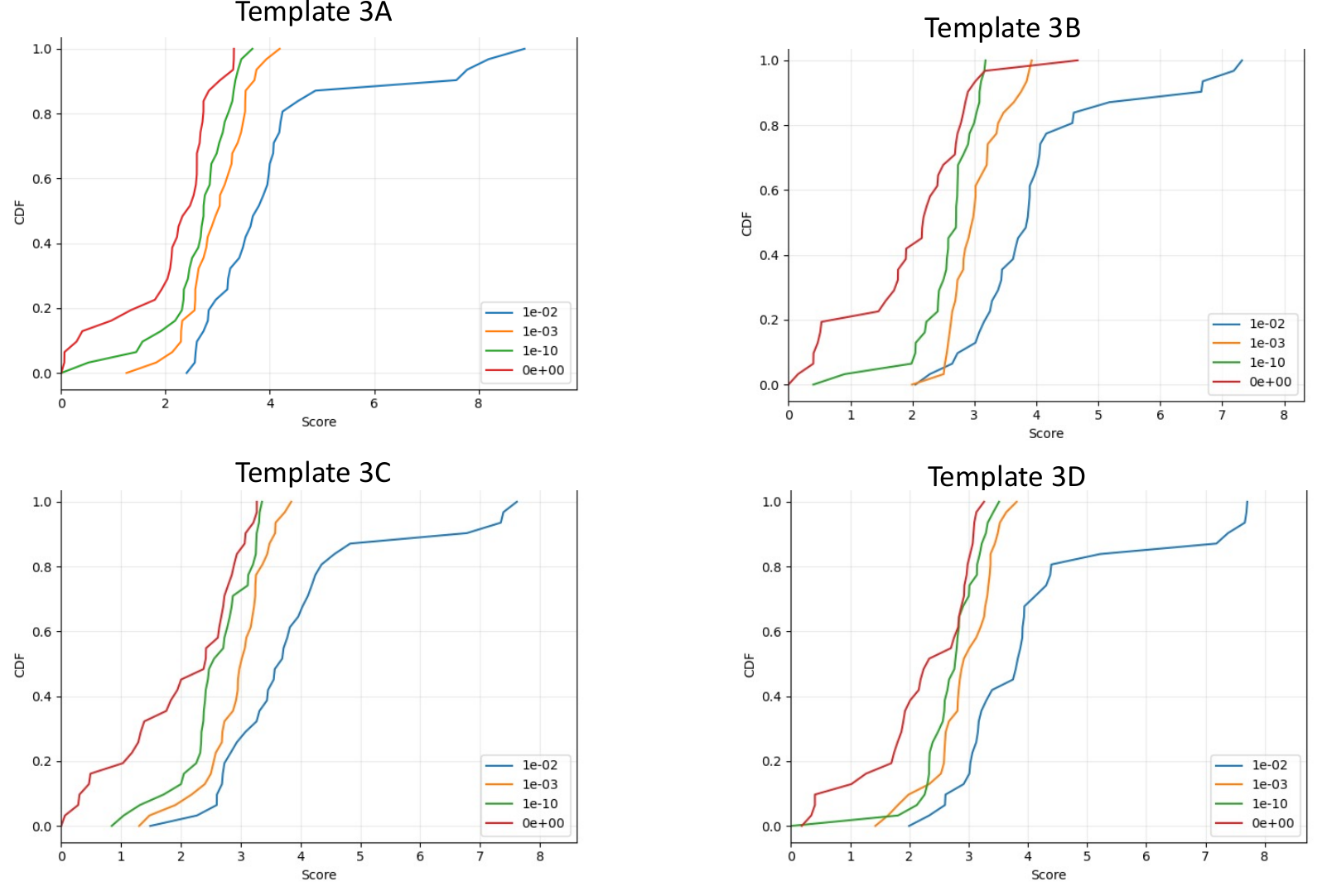}
    \caption{We run 32 random restarts of the \mg algorithm for each template recovery, plotting the empirical CDF of the GED of the recovered templates.
Different penalization values are represented with different colors in the plot.}
    \label{fig:tkb3}
\end{figure}

\vfill

\end{document}